\documentclass[letterpaper, 10 pt, conference]{ieeeconf}
\usepackage{amsmath,amsfonts}
\usepackage{algorithm}
\usepackage{array}
\usepackage{textcomp}
\usepackage{stfloats}
\usepackage{url}
\usepackage{verbatim}
\usepackage{graphicx}
\usepackage{cite}
\usepackage{relsize}
\usepackage[noend]{algpseudocode}
\usepackage{etoolbox}

\newbool{useappendix}
\booltrue{useappendix}

\usepackage{amsthm}
\usepackage{amssymb}
\usepackage{mathtools}
\usepackage{color}
\usepackage{cancel} 
\usepackage{hyperref}
\mathtoolsset{showonlyrefs}

\theoremstyle{definition}

\newtheorem{assumption}{\protect\assumptionname}
\newtheorem{lemma}{\protect\lemmaname}
\newtheorem{proposition}{\protect\propositionname}
\newtheorem{example}{Example}

\theoremstyle{plain}
\newtheorem{theorem}{\protect\theoremname}
\newtheorem{problem}{Problem}

\providecommand{\assumptionname}{Assumption}
\providecommand{\definitionname}{Definition}
\providecommand{\lemmaname}{Lemma}
\providecommand{\theoremname}{Theorem}
\providecommand{\propositionname}{Proposition}

\newcommand{\accentest}[1]{\hat{#1}}
\newcommand{\accentestprior}[1]{\overline{#1}}
\newcommand{\accentapx}[1]{\breve{#1}}
\newcommand{\accentapxprior}[1]{\widetilde{#1}}
\newcommand{\x}[1]{\mathbf{x}_{#1}}

\newcommand{\xest}[1]{\accentest{\mathbf{x}}_{#1}}
\newcommand{\xestb}[2]{\vectorBlock{\xest{#1}}{#2}}
\newcommand{\xestbo}[3]{\vectorBlockOwned{\xest{#1}}{#2}{#3}}

\newcommand{\xesto}[2]{\vectorOwned{\xest{#1}}{#2}}
\newcommand{\xestoi}[3]{\vectorOwnedIteration{\xest{#1}}{#2}{#3}}
\newcommand{\xestboi}[4]{\vectorBlockOwnedIteration{\xest{#1}}{#2}{#3}{#4}}
\newcommand{\xestprior}[1]{\accentestprior{\mathbf{x}}_{#1}}

\newcommand{\xestpriorbo}[3]{\vectorBlockOwned{\xestprior{#1}}{#2}{#3}}
\newcommand{\xestprioro}[2]{\vectorOwned{\xestprior{#1}}{#2}}
\newcommand{\xapx}[1]{\accentapx{\mathbf{x}}_{#1}}
\newcommand{\xapxb}[2]{\vectorBlock{\xapx{#1}}{#2}}
\newcommand{\xapxbo}[3]{\vectorBlockOwned{\xapx{#1}}{#2}{#3}}

\newcommand{\xapxoi}[3]{\vectorOwnedIteration{\xapx{#1}}{#2}{#3}}
\newcommand{\xapxboi}[4]{\vectorBlockOwnedIteration{\xapx{#1}}{#2}{#3}{#4}}
\newcommand{\xapxprior}[1]{\accentapxprior{\mathbf{x}}_{#1}}

\newcommand{\xapxpriorbo}[3]{\vectorBlockOwned{\xapxprior{#1}}{#2}{#3}}
\newcommand{\xapxprioro}[2]{\vectorOwned{\xapxprior{#1}}{#2}}

\newcommand{\xtotal}[1]{\mathbf{x}_{#1}}

\newcommand{\xtotali}[2]{\vectorIteration{\xtotal{#1}}{#2}}

\newcommand{\y}[1]{\mathbf{y}_{#1}}

\newcommand{\ybo}[3]{\vectorBlockOwned{\y{#1}}{#2}{#3}}
\newcommand{\yo}[2]{\vectorOwned{\y{#1}}{#2}}

\newcommand{\yi}[2]{\y{#2,#1}} 
\newcommand{\z}[1]{\mathbf{z}_{#1}}
\newcommand{\zapx}[1]{\accentapx{\mathbf{z}}_{#1}}
\newcommand{\procNoise}[1]{\mathbf{w}_{#1}}
\newcommand{\measNoise}[1]{\mathbf{v}_{#1}}
\newcommand{\measNoisei}[2]{\measNoise{#1}^{#2}}
\newcommand{\xUpdate}[3]{\vectorOwnedIteration{\theta_{#1}}{#2}{#3}}

\newcommand{\mapvar}[1]{\mathbf{X}_{#1}}

\newcommand{\errorSubopt}[2]{\mu_{#1}(#2)}
\newcommand{\cerrorSubopt}[1]{\mu_{#1}} 

\newcommand{\Pest}[1]{\accentest{\mathbf{P}}_{#1}}

\newcommand{\Pestprior}[1]{\accentestprior{\mathbf{P}}_{#1}}

\newcommand{\Papx}[1]{\accentapx{\mathbf{P}}_{#1}}
\newcommand{\Q}[1]{\mathbf{Q}_{#1}}
\newcommand{\Qinv}[1]{\mathbf{Q}_{#1}^{-1}}
\newcommand{\Rbf}[1]{\mathbf{R}_{#1}}

\newcommand{\Rinv}[1]{\mathbf{R}_{#1}^{-1}}
\newcommand{\infmat}{\mathbf{\Omega}}
\newcommand{\infmatest}[1]{\accentest{\infmat}_{#1}}
\newcommand{\infmatapx}[1]{\accentapx{\infmat}_{#1}}
\newcommand{\infmatestprior}[1]{\accentestprior{\infmat}_{#1}}
\newcommand{\infmatapxprior}[1]{\accentapxprior{\infmat}_{#1}}
\newcommand{\infmatA}[1]{\alpha_{#1}}
\newcommand{\infmatB}[1]{\beta_{#1}}
\newcommand{\infmatC}[1]{\gamma_{#1}}
\newcommand{\infmatD}[1]{\delta_{#1}}
\newcommand{\infmatcom}[1]{\Delta_{#1}}

\newcommand{\infmatapxcom}[1]{\Phi_{#1}}

\newcommand{\A}[1]{\mathbf{A}_{#1}}
\newcommand{\C}[1]{\mathbf{C}_{#1}}
\newcommand{\Ci}[2]{\C{#2,#1}}
\newcommand{\K}[1]{\mathbf{K}_{#1}}
\newcommand{\Kapx}[1]{\accentapx{\mathbf{K}}_{#1}}
\newcommand{\f}[1]{\mathbf{f}_{#1}}
\newcommand{\fapx}[1]{\accentapx{\mathbf{f}}_{#1}}
\newcommand{\capx}[1]{\accentapx{c}_{#1}}
\newcommand{\F}[1]{\mathbf{F}_{#1}}
\newcommand{\G}[1]{\mathbf{G}_{#1}}
\newcommand{\Pibf}[1]{\mathbf{\Pi}_{#1}}
\newcommand{\Hbf}[1]{\mathbf{H}_{#1}}
\newcommand{\W}[1]{\mathbf{W}_{#1}}

\newcommand{\Wapx}[1]{\accentapx{\mathbf{W}}_{#1}}
\newcommand{\Wapxinv}[1]{\accentapx{\mathbf{W}}_{#1}^{-1}}
\newcommand{\J}{J}
\newcommand{\Japx}{\accentapx{J}}
\newcommand{\coeffmaty}[1]{\mathbf{N}_{#1}}
\newcommand{\coeffmatxestprior}[1]{\mathbf{S}_{#1}}
\newcommand{\coeffmatxapxprior}[1]{\accentapx{\mathbf{S}}_{#1}}
\newcommand{\priorpropmat}[1]{\mathbf{L}_{#1}}
\newcommand{\U}[1]{\mathbf{U}_{#1}}
\newcommand{\infsub}[1]{\mathbf{V}_{#1}}
\newcommand{\infsubpos}[2]{\mathbf{Z}_{#1,#2}}
\newcommand{\infpriorpos}[1]{\mathbf{Y}_{#1}}

\newcommand{\norm}[1]{\left\lVert #1 \right\rVert}
\newcommand{\zero}[1]{\mathbf{0}^{#1}}
\newcommand{\eye}[1]{\mathbf{I}^{#1}}
\newcommand{\gaussian}[2]{\mathcal{N}\left(#1,#2\right)}
\newcommand{\gradb}[1]{\nabla_{#1}}
\newcommand{\E}[1]{\mathbb{E}\left[#1\right]}
\newcommand{\superth}[1]{#1^{\text{th}}}

\newcommand{\R}[1]{\mathbb{R}^{#1}}

\newcommand{\seq}[1]{[#1]}

\newcommand{\depsettir}[3]{\mathcal{D}^{#2}_{#3,#1}}
\newcommand{\depsetapxtir}[3]{\accentapx{\mathcal{D}}^{#2}_{#3,#1}}
\newcommand{\depsetapxZti}[2]{\accentapx{\mathcal{D}}^{#2,0}_{#1}}
\newcommand{\depsetapxtiq}[3]{\accentapx{\mathcal{D}}^{#2}_{#1,#3}}

\newcommand{\depop}[2]{\phi^{[#1]}\left(#2\right)}  

\newcommand{\vectorBlock}[2]{#1^{[#2]}}
\newcommand{\vectorBlockOwned}[3]{#1^{#3,[#2]}}

\newcommand{\vectorOwned}[2]{#1^{#2}}
\newcommand{\vectorOwnedIteration}[3]{#1^{#2}(#3)}
\newcommand{\vectorIteration}[2]{#1(#2)}
\newcommand{\vectorBlockOwnedIteration}[4]{#1^{#3,[#2]}(#4)}
\newcommand{\matrixRowBlock}[2]{#1^{[#2]}}
\newcommand{\matrixBlock}[3]{#1^{[#2][#3]}}
\newcommand{\B}[1]{\mathbf{B}^{[#1]}} 
\newcommand{\M}[1]{\mathbf{M}^{[#1]}} 
\newcommand{\Bs}[2]{\mathbf{B}^{[#1]}_{#2}} 
\newcommand{\Ms}[3]{\mathbf{M}^{[#1]}_{#2,#3}} 

\newcommand{\vbf}{\mathbf{v}}
\newcommand{\vbfb}[1]{\vectorBlock{\vbf}{#1}} 
\newcommand{\ubf}{\mathbf{u}}
\newcommand{\Wbf}{\mathbf{W}}
\newcommand{\Wbfrb}[1]{\matrixRowBlock{\Wbf}{#1}} 
\newcommand{\Wbfb}[2]{\matrixBlock{\Wbf}{#1}{#2}} 

\newcommand{\Tmax}{T}
\newcommand{\T}[1]{T_{#1}}
\newcommand{\sliwin}[1]{\tau(#1)} 
\newcommand{\psliwin}[1]{\pi(#1)} 
\newcommand{\maxIter}{k_{\max}}

\newcommand{\step}{\eta}
\DeclareMathOperator*{\argmax}{argmax}

\newcommand{\updateIterSet}[2]{\Psi_{#1}^{#2}}
\newcommand{\updateIterRef}[4]{\vectorBlockOwned{\sigma_{#1}}{#2}{#3}(#4)}
\newcommand{\cupdateIterRef}[3]{\vectorBlockOwned{\sigma_{#1}}{#2}{#3}}
\newcommand{\maxCalcs}{k_{\max}}

\ifbool{useappendix}{
\newcommand{\appref}[1]{See Appendix \ref{#1}.}
}{
\newcommand{\appref}[1]{See authors' technical report~\cite[Appendix \ref*{#1}]{pooley2026technical}.}
\includeonly{}
}

\definecolor{dark-yellow}{rgb}{0.6,0.6,0}
\definecolor{pink}{rgb}{1.0,0.2,0.8}
\newcommand{\ap}[1]{#1} 

\ifbool{useappendix}
{\title{\LARGE \bf
Technical Report: Asynchronous Distributed Trajectory \\ Estimation of Multi-Robot Systems
}}
{\title{\LARGE \bf
Asynchronous Distributed Trajectory \\ Estimation of Multi-Robot Systems
}}

\author{Adam Pooley and Matthew T. Hale$^1$
\thanks{$^1$ Authors are with the School of Electrical and Computer
Engineering, Georgia Institute of Technology, Atlanta, GA, USA.
Emails: \texttt{\{apooley3,mhale30\}@gatech.edu}
}
\thanks{
This work was supported by 
AFRL under grants FA8651-24-2-0003 and FA8651-23-F-A006 and
AFOSR under grant FA9550-19-1-0169. 
}
}

\IEEEoverridecommandlockouts
\begin{document}

\maketitle
\thispagestyle{empty}
\pagestyle{empty}

\begin{abstract}
Distributed trajectory estimation arises in many applications across robotics, but
existing implementations typically do not consider asynchrony in agents' communications
and computations. 
Therefore, we propose an asynchronous block coordinate descent algorithm for
distributed trajectory estimation. 
We consider a team of agents that observes a team of robots and estimates
the robots' states over a sliding window. 
The agents solve 
an approximation of the maximum a posteriori estimation problem, 
which we derive. We show
this approximation introduces negligible errors
and eliminates up to 96.9\% of communications among agents.
Next, we prove
that agents' iterates converge exponentially fast to the optimal estimate
of the robots' states. 
Simulations show that 
this approach has up to 64\%
less error than a comparable state-of-the-art algorithm.
Experiments on mobile robots show this approach 
is robust 
to delays whose lengths span three orders of magnitude.
\end{abstract}
\section{Introduction} \label{sect_introduction}
State estimation is a fundamental problem in autonomy, with applications ranging from mobile robotics and transportation to 
co-robots working alongside humans~\cite{thrun2005probabilistic, shorinwa2020distributed, khodayi2019distributed}.
State estimates can be computed with data from multiple sensors, and 
decentralized sensing offers several benefits,
including robustness to individual sensor failures 
and the ability to use different sensing modalities. 
When sensors are embedded in systems that can compute and communicate,
distributed state estimation (DSE) allows for collaborative estimation of the state of a system.
In this paper, 
we refer to the computing entities as ``agents'' and the observed robots as the ``targets'',
because estimation may be done by external processors rather than the robots themselves.

Existing works in the DSE literature~\cite{shorinwa2020distributed, olfati2005distributed, olfati2007distributed, shorinwa2024distributedtutorial, mateos2010distributed, shi2015extra, roumeliotis2002distributed} often
assume that agents compute and communicate synchronously. 
However, real-world systems often face computation and communication delays, e.g., due to 
adversarial jamming~\cite{priyadarshani2025jamming} or
saturated bandwidth~\cite{li2014communication}. 
Computation speeds can 
vary among agents due to heterogeneous hardware~\cite{nassralla2023hybrid, wang2025efficient}.
Attempts to synchronize agents' operations
can induce the ``straggler effect''~\cite{karakus2019redundancy},
which causes agents to only compute and communicate at the rate of the slowest among them. 
In such scenarios, synchronizing agents can significantly slow the convergence of 
multi-agent systems~\cite{dean2013tail}.

Therefore, in this paper we propose a decentralized state estimation algorithm that is 
designed to operate with asynchronous computations and communications.
We assume only that delays in agents' computations and communications are bounded (though the bound can be arbitrary),
which is called ``partial asynchrony\ap{''}~\cite{bertsekas1989parallel}. 
We consider state estimation for multi-robot systems, and 
agents solve a 
maximum \emph{a posteriori} (MAP) optimization problem whose solution
is the optimal state estimate of the target system. 
Our problem formulation 
estimates states over a sliding window of time, which
enables new measurements of the target system's outputs to be used to improve 
estimates of past states.

Our algorithm uses block coordinate descent (BCD) to jointly compute state estimates.
Each agent stores estimated values of all states of the target system onboard, but each agent 
computes 
estimates for only a subset of these states.
Agents communicate with each other to share updated values
of state estimates over time.
We emphasize that we do not simply apply BCD to a state estimation problem. 
Instead,
we show in a precise way how using BCD to solve the MAP problem
requires certain pairs of agents to communicate.
Then, we derive an approximation of the MAP problem
that substantially reduces required communications while inducing only negligible errors.

\subsection{Summary of Contributions}
The contributions of this work are as follows:
\begin{enumerate}
    \item We derive an approximated update law that maintains high accuracy while promoting sparsity of the required communications among agents 
    (Lemma~\ref{lemma_depsetapx}).
    \item We prove convergence to the solution of the approximated
    MAP problem and derive an exponential convergence rate under both asynchronous communications and asynchronous computations (Theorem~\ref{theorem_convergence_rate}).
    \item We present simulation results that compare our algorithm to a state-of-the-art algorithm
    and show an error reduction of up to 64\%
    (Section~\ref{sect_results_simulation}).
    \item We present hardware results that show the robustness of our algorithm under varying magnitudes of communication delays  (Section~\ref{sect_results_experiments}).
\end{enumerate}

\subsection{Related Works}
Different notions of asynchrony exist across the literature. 
We define asynchrony as
in the distributed optimization literature~\cite{bertsekas1989parallel},
which allows asynchronous computations and communications among agents. 
The development of asynchronous algorithms 
has been identified as 
a goal in the robotics literature~\cite{testa2025tutorial,jaleel2020distributed}, and
this paper contributes to that goal by developing a parallelized
asynchronous algorithm for state estimation. 
Related work in~\cite{ziegler2021distributed,xu2024d} applies distributed optimization techniques
to state estimation problems for certain classes of UAV swarms, though we differ by estimating
states of an arbitrary time-varying system. 
The closest prior work is~\cite{shorinwa2020distributed}. 
That work uses synchronous CADMM-based updates,
while the algorithm we present
uses asynchronous BCD updates.
We show in Section~\ref{sect_results} that
our algorithm attains up to a 64\% reduction
in estimation error relative to the algorithm from~\cite{shorinwa2020distributed}.

The rest of the paper is organized as follows. 
Section \ref{sect_problem_statement} provides a problem statement,
Section \ref{sect_synchronous_algorithm} analyzes communications, 
Section \ref{sect_approx_obj_reform} presents an approximation that promotes communication sparsity,
Section \ref{sect_asynch_algorithm} presents the algorithm,
Section \ref{sect_convergence} provides convergence analysis,
Section \ref{sect_results} presents simulations and experiments,
and Section \ref{sect_conclusion} concludes.

\subsection{Notation}
We use~$\mathbb{N}$ to denote the non-negative integers and~$\mathbb{N}^+$ to denote the positive integers.
We use $\norm{\cdot}$ to denote the Euclidean norm.
Given a symmetric, positive-definite matrix $Q\in\R{n\times n}$ and some $x\in\R{n}$,
we define $\norm{x}_Q = \sqrt{x^\top Q x}$.
We use $\zero{n\times m}$ to denote the matrix of zeros with $n$ rows and $m$ columns. 
We use $\eye{n\times n}$ to denote the~$n \times n$ identity matrix.
We use $\seq{d}$ to denote $\{1,\dots,d\}$ for $d\in\mathbb{N}^+$.
For matrices $A, B, C$, we write $\text{blkdiag}(A, B, C)$ for the 
block-diagonal matrix with~$A$, $B$, and~$C$ on its main diagonal. 
Any union of the form~$\bigcup_{i=a}^{b} S_i$ is empty if~$b < a$. 
\section{Problem Statement} \label{sect_problem_statement}
This section describes our system model and MAP formulation,
followed by a formal problem statement. 

\subsection{Multi-Agent Sensing Model}
\label{sect_multi-agent-system-model}

We consider $N \in \mathbb{N}^+$ agents, each 
of which measures outputs of the ``target system'',
which is a multi-robot system with state~$\x{t} \in \R{n}$ at time step $t \in \mathbb{N}$
for some~$n \in \mathbb{N}^+$. 
The state~$\x{t}$ contains the states of all robots in the system, and it 
has the time-varying dynamics
\begin{equation}
    \label{eq_state_transition_model}
    \x{t+1} = \A{t}\x{t} + \procNoise{t},
\end{equation}
where $\A{t}\in\R{n\times n}$
and the process noise $\procNoise{t}\in\R{n}$ obeys $\procNoise{t}\sim\gaussian{\zero{}}{\Q{t}}\in\R{n}$ with $\Q{t} = \Q{t}^\top \succ 0$.

For each~$i \in [N]$, 
agent $i$'s measurement at time~$t$
is denoted $\yi{t}{i}\in\R{m_i}$ 
with~$m_i \in \mathbb{N}^+$, and it is given by 
    $\yi{t}{i} = \Ci{t}{i} \x{t} + \measNoisei{t}{i}$,
where $\Ci{t}{i}\in\R{m_i\times n}$ and~$\measNoisei{t}{i}\in\R{m_i}$ is measurement noise. 
We set 
$\y{t} = \left[ (\yi{t}{1})^\top \cdots (\yi{t}{N})^\top \right]^\top \in \R{m}$,
$\C{t} = \left[ (\Ci{t}{1})^\top \cdots (\Ci{t}{N})^\top \right]^\top \in \R{m\times n}$, along with 
$\measNoise{t} = \left[ (\measNoisei{t}{1})^\top \cdots (\measNoisei{t}{N})^\top \right]^\top \in \R{m}$,
with $m = \sum_{i=1}^N m_i$. 
Then
\begin{equation}
    \label{eq_measurement_model}
    \y{t} = \C{t} \x{t} + \measNoise{t}
\end{equation}
models measurements from all agents at time~$t$. 
We consider measurement noise~$\measNoise{t}\sim \gaussian{\zero{}}{\Rbf{t}}$,
where~$\Rbf{t} = \Rbf{t}^\top \succ 0$. 
We allow~$\Rbf{t}$ to be non-block diagonal, which allows
measurements to be correlated among agents.

\subsection{MAP Formulation}
\label{sect_map_formulation}
We first consider 
the objective function 
from~\cite{shorinwa2020distributed}, and we modify it
in Section~\ref{sect_approx_obj_reform}
to make it more amenable to decentralized implementation.

Consider estimating the state trajectory of a multi-robot
system across a time horizon from time~$0$ to some time~$t$. 
We denote the optimal state estimate of the system by
$\xest{0:t} := [\xest{0}^\top \, \cdots \, \xest{t}^\top]^\top \in \mathbb{R}^{n(t+1)}$,
where~$\xest{\ell} \in \mathbb{R}^n$ is the estimated state at time~$\ell \in \{0, \ldots, t\}$.
We model the system's initial state $\x{0}$
as a Gaussian distributed random variable with mean $\xestprior{0}$ and covariance $\Pestprior{0}$. 
The optimal estimate $\xest{0:t}$
can be computed in terms of
a prior for the initial state, 
namely $\xestprior{0}$, 
and all measurements up to time $t$, namely~$\y{1:t} = [\y{1}^\top \, \cdots \, \y{t}^\top]^\top$. 
Mathematically, the MAP estimate~$\xest{0:t}$ of the
trajectory~$\x{0:t}$ 
is given by 
\begin{equation}
    \xest{0:t} 
    \label{eq_map_total}
    \!=\! \argmax_{\mapvar{0:t}} p(\mapvar{0} \!\mid\! \xestprior{0}) \prod_{\ell=0}^{t-1} p(\mapvar{\ell+1} \!\mid\! \mapvar{\ell}) \prod_{\ell=1}^{t}p(\y{\ell} \!\mid\! \mapvar{\ell}),
\end{equation}
where the conditional PDFs are Gaussians. Then, 
the posterior PDF of $\x{0:t}$ is also Gaussian, 
which we represent in closed form using the mean $\xest{0:t}$ and the covariance $\Pest{0:t}$.  

We are interested in using current measurements to 
improve the accuracy of state estimates from 
up to $\Tmax \in \ap{\mathbb{N}^+}$ time steps prior to the current time. 
Therefore, 
we consider a sliding window approach in which measurements at time~$t$ can be used to improve state estimates 
over the sliding window~$\{t - \T{t}, \ldots, t\}$, 
where~$\T{t} = \min\{t,\Tmax\}$.
The estimated trajectory over this window is denoted by
$\xest{t-\T{t}:t} = \begin{bmatrix}
    \xest{t-\T{t}}^\top & \cdots & \xest{t}^\top
\end{bmatrix}^\top$.

We begin with 
a prior belief over the sliding window~$\{t-\T{t}, \dots, t-1\}$, 
which is Gaussian distributed 
with mean $\xestprior{t-\T{t}:t-1}$ and covariance $\Pestprior{t-\T{t}:t-1}\succ0$. 
We calculate the estimate $\xest{t-\T{t}:t}$ by maximizing the posterior probability of the estimate over 
the current sliding window, namely $\x{t-\T{t}:t}$, 
conditioned on the prior estimate $\xestprior{t-\T{t}:t-1}$ 
and the current measurement $\y{t}$. This maximization is written as
\begin{multline}
    \label{eq_map_sliding_window}
    \xest{t-\T{t}:t} = \argmax_{\x{t-\T{t}:t}}
    p(\x{t-\T{t}:t-1}\mid\xestprior{t-\T{t}:t-1})\\
    \cdot p(\x{t}\mid\x{t-1}) p(\y{t}\mid\x{t}).
\end{multline}

We use \eqref{eq_state_transition_model} and \eqref{eq_measurement_model}
to write an objective function whose minimizer is the solution to \eqref{eq_map_sliding_window},
given by
\begin{multline}
    \label{eq_obj_func_true_norms}
    \J(\xest{t-\T{t}:t}) = \norm{\xest{t}-\A{t-1}\xest{t-1}}_{\Q{t-1}^{-1}}^2
    + \norm{\y{t}-\C{t}\xest{t}}_{\Rbf{t}^{-1}}^2 \\
    + \norm{\xest{t-\T{t}:t-1} - \xestprior{t-\T{t}:t-1}}_{\Pestprior{t-\T{t}:t-1}^{-1}}^2.
\end{multline}

We define 
\begin{align}
    \F{t} \!&=\!\begin{bmatrix}
        \zero{n\times n(\T{t}-1)}  & \hspace{-6pt}-\A{t-1} & \! \eye{n\times n}
    \end{bmatrix}\!,
     \G{t} \!=\! \begin{bmatrix}
        \zero{m\times n\T{t}} & \C{t}
    \end{bmatrix}
    \\
    \Pibf{t} &= \begin{bmatrix}
        \eye{n\T{t} \times n\T{t}} & \zero{n\T{t} \times n}
    \end{bmatrix}, 
    \, 
    \label{eq_def_H}
    \Hbf{t} = \begin{bmatrix}
        \F{t}^\top & \G{t}^\top & \Pibf{t}^\top
    \end{bmatrix}^\top
    \\
    \z{t} &= \begin{bmatrix}
        \zero{1\times n} & \y{t}^\top & \xestprior{t-\T{t}:t-1}^\top
    \end{bmatrix}^\top
    \\
    \label{eq_def_W}
    \W{t} &= \text{blkdiag}\left(\Q{t-1}, \Rbf{t}, \Pestprior{t-\T{t}:t-1}\right), 
\end{align}
and these matrices allow us to express~$J$ in \eqref{eq_obj_func_true_norms} as
\begin{equation}
    \label{eq_obj_func_true_quad}
    \J(\xest{t-\T{t}:t}) = \frac{1}{2} \xest{t-\T{t}:t}^\top \K{t} \xest{t-\T{t}:t} + \f{t}^\top \xest{t-\T{t}:t} + c_t,
\end{equation}
where
    $\K{t} = 2\Hbf{t}^\top \W{t}^{-1} \Hbf{t}$,      
    $\f{t} = -2\Hbf{t}^\top \W{t}^{-1} \z{t}$, and
    ${c_t = \z{t}^\top \W{t}^{-1} \z{t}}$.

We define the sequences $\sliwin{t} = \{t-\T{t},\dots,t\}$ and $\psliwin{t} = \{t-\T{t},\dots,t-1\}$
so that we can denote terms like $\xest{t-\T{t}:t}$ and $\xestprior{t-\T{t}:t-1}$ by
$\xest{\sliwin{t}}$ and $\xestprior{\psliwin{t}}$, respectively.

\subsection{Problem Statement}
The following problem is the focus of the rest of the paper. 

\begin{problem}
    \label{problem_synch}
    Consider $N$ agents observing a target system with dynamics given by \eqref{eq_state_transition_model} and measurement model given by \eqref{eq_measurement_model}.
    Given a prior distribution with mean $\xestprior{\psliwin{t}}$ and covariance $\Pestprior{\psliwin{t}}$, 
    over a sliding window $\{t-\T{t},\dots,t\}$ solve
    \begin{equation}
        \xest{\sliwin{t}} = \underset{\x{\sliwin{t}}}{\text{minimize}} \hspace{1em} \J(\x{\sliwin{t}}),
    \end{equation}
    for $t=1,2,\dots$, where $\J$ is given by \eqref{eq_obj_func_true_quad}.
\end{problem}

The solution to Problem~\ref{problem_synch} will be the estimate $\xest{\sliwin{t}}$.
After solving this problem at time $t$, the sliding window is advanced forward by one time step and the problem is reinitialized at time $t+1$, which requires a new prior distribution.
By defining
\begin{equation} 
    \label{eq_def_U}
    \U{t} = \begin{bmatrix}
        \zero{n\T{t+1}\times n(\T{t}-\T{t+1}+1)} & \eye{n\T{t+1}\times n\T{t+1}}
    \end{bmatrix},
\end{equation}
we relate the prior at time~$t+1$
to the sliding window estimate from time~$t$ via 
\begin{align}
    \label{eq_prior_update}
    \xestprior{\psliwin{t+1}} &= \U{t} \xest{\sliwin{t}},
    \\
    \label{eq_cov_prior_update}
    \Pestprior{\psliwin{t+1}} &= \U{t} \Pest{\sliwin{t}} \U{t}^\top.
\end{align}

In words, \eqref{eq_prior_update} and \eqref{eq_cov_prior_update} calculate the prior at time $t+1$ while accounting for the growth of the sliding window.
We note that when $t<T$, the sliding window grows and the entire previous estimate is included in the next prior.
Mathematically, if $t<T$, then $\U{t} = \eye{n\T{t+1}\times n\T{t+1}}$,
the relation in 
\eqref{eq_prior_update} reduces to $\xestprior{\psliwin{t+1}} = \xest{\sliwin{t}}$,
and \eqref{eq_cov_prior_update} reduces to $\Pestprior{\psliwin{t+1}} = \Pest{\sliwin{t}}$.

\section{Problem-Induced Communications} \label{sect_synchronous_algorithm}
This section presents a synchronous algorithm to solve Problem~\ref{problem_synch}
to determine the required communications among agents.
By showing that agents will generally require all-to-all communications,
we motivate an approximation
to the prior update law that 
reduces communications while
introducing negligible estimation error, which we derive in Section~\ref{sect_approx_obj_reform}.
Then we present an asynchronous algorithm in Section~\ref{sect_asynch_algorithm}.


\subsection{Block Notation}
We present notation for manipulating blocks of vectors and matrices.
Consider a vector $\vbf\in\R{r}$ and a matrix $\Wbf\in\R{r\times q}$, 
where $r=\sum_{i=1}^N r_i$ and $q=\sum_{i=1}^N q_i$.
We denote the ``$\superth{i}$ block'' of $\vbf$ as $\vbfb{i}\in\R{r_i}$, so that
$\vbf = \begin{bmatrix}
    \left(\vbfb{1}\right)^\top & \cdots & \left(\vbfb{N}\right)^\top
\end{bmatrix}^\top$.
We denote the ``$\superth{i}$ row block'' of $\Wbf$ as 
$\Wbfrb{i}\in\R{r_i\times q}$, so that 
$\Wbf = \begin{bmatrix}
    \left(\Wbfrb{1}\right)^\top & \cdots & \left(\Wbfrb{N}\right)^\top
\end{bmatrix}^\top$,
and we denote the ``$\superth{i}\superth{j}$ block'' of $\Wbf$ as $\Wbfb{i}{j}\in\R{r_i\times q_i}$, so that
$\Wbfrb{i} = \begin{bmatrix}
    \Wbfb{i}{1} & \cdots & \Wbfb{i}{N}
\end{bmatrix}$.
We define the block slicing matrix $\Bs{i}{r} \in \R{r_i \times r}$
in terms of $\vbfb{i} = \Bs{i}{r} \vbf$
according to
\begin{equation}
    \Bs{i}{r} = \begin{bmatrix}
        \zero{r_i\times r_1} & 
        \hspace{-8pt}\cdots & 
        \hspace{-8pt}\zero{r_i\times r_{i-1}} & 
        \hspace{-8pt}\eye{r_i \times r_i} &
        \hspace{-8pt}\zero{r_i\times r_{i+1}} & 
        \hspace{-8pt}\cdots & 
        \hspace{-8pt}\zero{r_i\times r_N}
    \end{bmatrix},
\end{equation}
where
$\Wbfrb{i} = \Bs{i}{r} \Wbf$ and
$\Wbfb{i}{j} = \Bs{i}{r} \Wbf (\Bs{j}{q})^\top$.

\subsection{Trajectory Block Notation}
To apply these ideas to concatenations of vectors and matrices, we define the concatenated block slicing matrix
\begin{equation}
    \label{eq_def_M}
    \Ms{i}{r}{\ell} = \eye{\ell\times\ell} \otimes \Bs{i}{r}.
\end{equation}

Consider a concatenated vector 
$\vbf_{1:\ell} = [\vbf_{1}^\top \cdots \vbf_{\ell}^\top]^\top \in \R{\ell r}$,
where $\vbf_{i}\in\R{r}$ for all $i\in[\ell]$,
and a concatenated matrix
$\Wbf_{1:\ell} = \begin{bmatrix}
    \Wbf_{1,1} & \cdots & \Wbf_{1,\ell} \\
    \vdots & \ddots & \vdots \\
    \Wbf_{\ell,1} & \cdots & \Wbf_{\ell,\ell}
\end{bmatrix} \in \R{\ell r \times \ell q}$,
where $\Wbf_{i,j}\in\R{r\times q}$ for all $i,j \in [\ell]$.
We denote the ``$\superth{i}$ block'' of $\vbf_{1:\ell}$ as $\vbfb{i}_{1:\ell} \in \R{\ell r_i}$, so that $\vbfb{i}_{1:\ell} = \Ms{i}{r}{\ell} \vbf_{1:\ell}$.
We denote the ``$\superth{i}$ row block'' of $\Wbf_{1:\ell}$ as $\Wbfrb{i}_{1:\ell} \in \R{\ell r_i \times \ell q}$, so that $\Wbfrb{i}_{1:\ell} = \Ms{i}{r}{\ell} \Wbf_{1:\ell}$, and
we denote the ``$\superth{i}\superth{j}$ block'' of $\Wbf_{1:\ell}$ as $\Wbfb{i}{j}_{1:\ell} \in \R{\ell r_i \times \ell q_i}$, 
so that $\Wbfb{i}{j}_{1:\ell} = \Ms{i}{r}{\ell} \Wbf_{1:\ell} (\Ms{j}{q}{\ell})^\top$.


For brevity, $\Bs{i}{r}$ and $\Ms{i}{r}{\ell}$ are often written as $\B{i}$ and $\M{i}$,
where $r$ and $\ell$ are implied by the dimensions of the matrices and vectors they multiply.

\subsection{Block Coordinate Descent}
We first describe some of the computations BCD must perform to solve Problem~\ref{problem_synch}. 
Agent $i$ computes the $\superth{i}$ block of the estimate,
which is 
$\xestb{\sliwin{t}}{i} = \M{i} \xest{\sliwin{t}}$.
Estimation is performed via iterative rounds of computing new estimates and communicating them with other agents.
The times at which these operations are performed are
indexed by an iteration counter $k$.
For a fixed $t$, the algorithm initializes at $k=0$ and performs some number of iterations $\maxIter \geq 1$.
Both~$k$ and~$\maxIter$ are introduced only for analysis. The agents do not need to know
the values of~$\maxIter$ or~$k$.




Agent $i$ stores three quantities onboard:
(i) $\yo{t}{i}$, which is agent $i$'s local copy of all measurements 
of the target system at time $t$,
(ii) $\xestoi{\sliwin{t}}{i}{k}$, which is agent $i$'s local copy of the trajectory estimate $\xest{\sliwin{t}}$ 
at iteration $k$, and 
(iii) $\xestprioro{\psliwin{t}}{i}$, which is agent $i$'s local copy of the prior $\xestprior{\psliwin{t}}$.
These variables are listed in Table~\ref{table_vars_synch}.
\begin{table}[tb]
    \centering
    \caption{Symbols and definitions for agents' onboard information}
    \label{table_vars_synch}
    \begin{tabular}{|m{0.27\columnwidth}|m{0.58\columnwidth}|}
        \hline
        Symbol & Definition 
        \\ \hline
        $\xesto{\sliwin{t}}{i} \in \R{n(\T{t}+1)}$ 
        & Agent $i$'s local copy of the estimate $\xest{\sliwin{t}}$ 
        \\ \hline
        $\yo{t}{i} \in \R{m}$
        & Agent $i$'s local copy of the measurement $\y{t}$ 
        \\ \hline
        $\xestprioro{\psliwin{t}}{i} \in \R{n\T{t}}$ 
        & Agent $i$'s local copy of the prior $\xestprior{\psliwin{t}}$ 
        \\ \hline
        $\xestbo{\sliwin{t}}{i}{i} \in \R{n_i(\T{t}+1)}$ 
        & Agent $i$'s local copy of the $\superth{i}$ block \newline of the estimate $\xest{\sliwin{t}}$ 
        \\ \hline
        $\ybo{t}{i}{i} \in \R{m_i}$
        & Agent $i$'s local copy of the $\superth{i}$ block \newline of the measurement $\y{t}$ 
        \\ \hline
        $\xestpriorbo{\psliwin{t}}{i}{i} \in \R{n_i\T{t}}$ 
        & Agent $i$'s local copy of the $\superth{i}$ block \newline of the prior $\xestprior{\psliwin{t}}$ 
        \\ \hline
    \end{tabular}
\end{table}

At the start of time step $t$, each agent takes a measurement of the target system, which for agent $i$ \ap{was previously denoted $\yi{t}{i}$. For consistent use with block notation, agent $i$'s measurement will be denoted by $\ybo{t}{i}{i}$}.
Then, an initial round of communication occurs in which, for each $i\in\seq{N}$, agent $i$ broadcasts their measurement $\ybo{t}{i}{i}$ and prior estimate $\xestpriorbo{\psliwin{t}}{i}{i}$.
Agent $i$ receives $\ybo{t}{j}{j}$ from agents $j\in\seq{N}\setminus \{i\}$ and agent $i$ receives $\xestpriorbo{\psliwin{t}}{j}{j}$ from agents $j\in\seq{N}\setminus \{i\}$. 
Agent $i$ sets $\ybo{t}{j}{i} = \ybo{t}{j}{j}$ and $\xestpriorbo{\psliwin{t}}{j}{i} = \xestpriorbo{\psliwin{t}}{j}{j}$ for all $j\in\seq{N}\setminus\{i\}$.

For all $i\in\seq{N}$, agent $i$ initializes time step~$t$'s estimate $\xestboi{\sliwin{t}}{i}{i}{0}$ using the prior estimate $\xestprioro{\psliwin{t}}{i}$ and the $\superth{i}$ block of the transition model \eqref{eq_state_transition_model}, according to
\begin{equation}
    \label{eq_alg_prior_propagate}
    \xestboi{\sliwin{t}}{i}{i}{0} = \matrixRowBlock{\priorpropmat{t}}{i} \xestprioro{\psliwin{t}}{i}, \  
    \priorpropmat{t} = \begin{bmatrix}
        \eye{n\T{t} \times n\T{t}} \\
        \begin{matrix}
            \zero{n\times n(\T{t}-1)} & \A{t-1}
        \end{matrix}
    \end{bmatrix}.
\end{equation}
This step forms the initial trajectory estimate with the prior estimate $\xestpriorbo{\psliwin{t}}{i}{i}$,\
along with its final state value propagated through the dynamics, given by $\matrixRowBlock{\A{t-1}}{i} \xestprioro{t-1}{i}$.
Agents communicate and compute for $\maxIter$ iterations for each timestep~$t$.

At certain times, agent $i$ receives $\xestboi{\sliwin{t}}{j}{j}{k}$ 
from agent $j$ for $j\in\seq{N}\setminus\{i\}$.
Agent $i$ overwrites its local copy of $\xestb{\sliwin{t}}{j}$ by setting $\xestboi{\sliwin{t}}{j}{i}{k} = \xestboi{\sliwin{t}}{j}{j}{k}$.
Agent $i$ uses $\xesto{\sliwin{t}}{i}$ to perform a block gradient update with step size $\step$, given by
\begin{equation}
    \label{eq_bcd_update_gradient}
    \xestboi{\sliwin{t}}{i}{i}{k+1}
    = \xestboi{\sliwin{t}}{i}{i}{k}
    - \step \cdot \gradb{i} \J(\xestoi{\sliwin{t}}{i}{k}),
\end{equation}
where $\gradb{i} \J$ denotes the $\superth{i}$ block of the gradient of $\J$,
i.e., $\gradb{i} \J = \M{i} \nabla \J$, where $\nabla J$ is calculated as
\begin{equation}
    \label{eq_gradient_true}
    \nabla \J(\xest{\sliwin{t}}) = \K{t} \xest{\sliwin{t}} + \f{t} 
\end{equation}
and $\M{i}$ is from \eqref{eq_def_M}.
The step size $\step$ is constant for all agents and 
will be discussed in Section~\ref{sect_convergence}.
We define
\begin{equation} \label{eq_coeffmaty}
    \coeffmaty{t} = -2\G{t}^\top \Rinv{t} \quad \textnormal{ and } \quad 
    \coeffmatxestprior{t} = -2\Pibf{t}^\top \Pestprior{\psliwin{t}}^{-1},
\end{equation}
which we use to express \eqref{eq_gradient_true} as
\begin{equation}
    \label{eq_gradient_true_explicit}
    \nabla J(\xest{\sliwin{t}}) = \K{t} \xest{\sliwin{t}} + \coeffmaty{t} \y{t} + \coeffmatxestprior{t} \xestprior{\psliwin{t}}.
\end{equation}
Using \eqref{eq_gradient_true_explicit} in \eqref{eq_bcd_update_gradient} yields
\begin{multline}
    \label{eq_alg_bcd_update_synch}
    \xestboi{\sliwin{t}}{i}{i}{k+1}
    = \xestboi{\sliwin{t}}{i}{i}{k}
    \\
    - \eta \matrixRowBlock{\K{t}}{i} \xestoi{\sliwin{t}}{i}{k}
    - \eta \matrixRowBlock{\coeffmaty{t}}{i} \yo{t}{i}
    - \eta \matrixRowBlock{\coeffmatxestprior{t}}{i} \xestprioro{\psliwin{t}}{i}.
\end{multline}

After completing $\maxIter$ iterations for time step $t$, agent $i$ calculates its next prior estimate, $\xestpriorbo{\psliwin{t+1}}{i}{i}$, using their final state estimate from time step~$t$, 
$\xestboi{\sliwin{t}}{i}{i}{\maxIter}$, 
and the $\superth{i}$ block of \eqref{eq_prior_update}, according to
    $\xestpriorbo{\psliwin{t+1}}{i}{i} = \matrixBlock{\U{t}}{i}{i} \xestboi{\sliwin{t}}{i}{i}{\maxIter}$.
We emphasize that~$\xestboi{\sliwin{t}}{i}{i}{\maxIter}$ is simply the last iterate generated by
agent~$i$ when estimating~$\x{\sliwin{t}}$. 

\subsection{Induced Communications} \label{sect_induced_comms_est}
Each agent's update law contains terms that must be sent to it by other agents. 
To analyze required communications, we express \eqref{eq_alg_prior_propagate} and \eqref{eq_alg_bcd_update_synch},
respectively, 
%
%
as 
\begin{align}
    \label{eq_alg_prior_propagate_sum}
    \xestboi{\sliwin{t}}{i}{i}{0} &= \sum_{j=1}^N \matrixBlock{\priorpropmat{t}}{i}{j} \xestpriorbo{\psliwin{t}}{j}{i}
    \\
    \xestboi{\sliwin{t}}{i}{i}{k+1}
    &= \xestboi{\sliwin{t}}{i}{i}{k}
    - \eta \sum_{j=1}^N \matrixBlock{\K{t}}{i}{j} \xestboi{\sliwin{t}}{j}{i}{k}
    \\
    \label{eq_alg_bcd_update_synch_sum}
    -& \eta \sum_{j=1}^N \matrixBlock{\coeffmaty{t}}{i}{j} \ybo{t}{j}{i}
    - \eta \sum_{j=1}^N \matrixBlock{\coeffmatxestprior{t}}{i}{j} \xestpriorbo{\psliwin{t}}{j}{i}. 
\end{align}

To minimize the induced communications, we seek to eliminate
as many terms as possible from the sums in
\eqref{eq_alg_prior_propagate_sum} and \eqref{eq_alg_bcd_update_synch_sum}. 
We can remove any terms whose
corresponding $\superth{i}\superth{j}$
matrix block equals zero.
We introduce the dependency operator on the $\superth{i}$ row block of a matrix $\mathbf{A}$ as
$\depop{i}{\mathbf{A}} = \left\{j: (\mathbf{A})^{[i][j]} \neq \zero{}\right\}$,
which 
contains the indices of nonzero column blocks of $\matrixRowBlock{\mathbf{A}}{i}$.

We formally define the dependency sets 
\begin{multline}
    \label{eq_depsets_synch}
    \depsettir{t}{i}{\A{}} = \depop{i}{\priorpropmat{t}}, \quad
    \depsettir{t}{i}{\xest{}} = \depop{i}{\K{t}}, \\
    \depsettir{t}{i}{\y{}} = \depop{i}{\coeffmaty{t}}, \quad 
    \depsettir{t}{i}{\xestprior{}} = \depop{i}{\coeffmatxestprior{t}},
\end{multline}
where $\depsettir{t}{i}{\A{}}$ encodes the blocks of $\xestprioro{\psliwin{t}}{i}$ needed to initialize the current estimate, 
and the sets $\depsettir{t}{i}{\xest{}}$, $\depsettir{t}{i}{\y{}}$, and $\depsettir{t}{i}{\xestprior{}}$, respectively, encode the blocks of $\xesto{\sliwin{t}}{i}$, $\yo{t}{i}$, and $\xestprioro{\psliwin{t}}{i}$ needed to calculate the next gradient update.
Using these sets, we express \eqref{eq_alg_prior_propagate_sum} as $\xestboi{\sliwin{t}}{i}{i}{0} = \sum_{j\in\depsettir{t}{i}{\A{}}} \matrixBlock{\priorpropmat{t}}{i}{j} \xestpriorbo{\psliwin{t}}{j}{i}$ and \eqref{eq_alg_bcd_update_synch_sum} as
\begin{multline}
    \label{eq_alg_bcd_update_synch_minsum}
    \xestboi{\sliwin{t}}{i}{i}{k+1}
    = \xestboi{\sliwin{t}}{i}{i}{k}
    - \eta \sum_{j\in\depsettir{t}{i}{\xest{}}} \matrixBlock{\K{t}}{i}{j} \xestboi{\sliwin{t}}{j}{i}{k}
    \\
    - \eta \sum_{j\in\depsettir{t}{i}{\y{}}} \matrixBlock{\coeffmaty{t}}{i}{j} \ybo{t}{j}{i}
    - \eta \sum_{j\in\depsettir{t}{i}{\xestprior{}}} \matrixBlock{\coeffmatxestprior{t}}{i}{j} \xestpriorbo{\psliwin{t}}{j}{i}.
\end{multline}

To compute these expressions, 
agent $i$ must receive
$\xestpriorbo{\psliwin{t}}{j}{j}$ from agents with indices $j \in \depsettir{t}{i}{\A{}}\cup\depsettir{t}{i}{\xestprior{}}$, the measurement
$\ybo{t}{j}{j}$ from agents with indices $j \in \depsettir{t}{i}{\y{}}$, and 
$\xestboi{\sliwin{t}}{j}{j}{k}$ from agents with indices $j \in \depsettir{t}{i}{\xest{}}$.
We define the information matrices
$\infmatestprior{0:0} = \Pestprior{0:0}^{-1}$,
$\infmatest{\sliwin{t}} = \Pest{\sliwin{t}}^{-1}$, and
$\infmatestprior{\psliwin{t}} = \Pestprior{\psliwin{t}}^{-1}$.

\begin{proposition}
    \label{proposition_dep_synch}
    Let $\depsettir{t}{i}{\xest{}} = \depop{i}{\K{t}}$ be defined from \eqref{eq_depsets_synch}.
    Then for $t\geq T$, we have $\depop{i}{\infmatA{t}} \subseteq \depsettir{t}{i}{\xest{}}$,
    where, for $t=T$,
        $\infmatA{t} = \infmatestprior{0:0} + \A{0}^\top \Q{0}^{-1} \A{0}$,
    and for $t>T$, $\infmatA{t}$ is recursively calculated according to
    \begin{multline}
        \label{eq_infmatA_recursive}
        \infmatA{t} = \C{t-\T{}}^\top \Rinv{t-\T{}} \C{t-\T{}} + \A{t-\T{}}^\top \Q{t-\T{}}^{-1} \A{t-\T{}}
        \\
        + \Q{t-\T{}-1}^{-1} - \Q{t-\T{}-1}^{-1} \A{t-\T{}-1} \infmatA{t-1}^{-1} \A{t-\T{}-1}^\top \Q{t-\T{}-1}^{-1}.
    \end{multline}
\end{proposition}

\begin{proof}
    \appref{subsect_proposition_dep_synch}
\end{proof}

Proposition~\ref{proposition_dep_synch} shows that $\infmatA{t}$ is recursively calculated as a function of $\infmatA{t-1}^{-1}$ and all other model terms $\A{t-\T{}}$, $\Q{t-\T{}}$, $\C{t-\T{}}$, $\Rbf{t-\T{}}$, and initialized using $\Pestprior{0:0}$.
The presence of the inverse of $\infmatA{t-1}$ in this recursive definition generally does not maintain sparsity, 
in the sense that~$\alpha_{t-1}^{-1}$ may have many non-zero entries even when~$\alpha_{t-1}$ has few non-zero entries. 
In fact, $\infmatA{t}$ may be completely dense, i.e., it may have no zero entries. 
In that case, the dependency set $\depsettir{t}{i}{\xest{}}$ would equal the set of 
\ap{all agents $\seq{N}$,}
requiring that agent $i$ performs all-to-all communication.
To illustrate this phenomenon, we consider the following example.

\begin{example}
    \label{example_toy_synch}
    We consider an example with~$10$ agents, where $\T{}=5$, 
    $n_i = m_i = 2$ for all agents $i\in\seq{10}$,
    $\C{t} = \Q{t} = \Rbf{t} = \Pestprior{0:0} = \eye{20 \times 20}$,
    and $\A{t}\in\R{20\times 20}$ is 
    \begin{equation}
        \label{eq_example_A_bidiagonal}
        \A{t} = \begin{bmatrix}
            1 & 1 & \cdots & 0 \\
            0 & 1 & \ddots & \vdots \\
            \vdots & \ddots & \ddots & 1 \\
            0 & \cdots & 0 & 1
        \end{bmatrix}.
    \end{equation}
    In this example, $\A{t}$ is the only model term to present inter-agent dependencies because all other matrices are diagonal. 
    Simulating this problem shows that after only 6 time steps, the matrix $\infmatA{\T{}+1}$ from \eqref{eq_infmatA_recursive} is completely dense, with all 400 elements being nonzero.
    Because $\infmatA{t}$ appears in $\infmatest{t}$
    which determines agent communication dependencies,
    we can see that $\depsettir{t}{i}{\xest{}} = \seq{N}$,     
    and these agents will have to perform all-to-all communications to 
    perform their calculations.     
\end{example}


\section{Approximate Objective Formulation}
\label{sect_approx_obj_reform}
In this section, 
\ap{
we address the issue of all-to-all communication by introducing an 
approximate update 
that promotes sparse communication. 
}
\ap{For $t\geq 1$, we can calculate an approximate prior information matrix,}
denoted $\infmatapxprior{\psliwin{t}}$,
via 
\begin{equation}
    \label{eq_infmatapxprior_vibe}
    \infmatapxprior{\psliwin{t+1}} = \U{t} \infmatest{\sliwin{t}} \U{t}^\top.
\end{equation}
\ap{Using $\infmatapxprior{\psliwin{t+1}}$ admits a new approximate objective and solution.}
Following the same analysis in Section~\ref{sect_map_formulation},
we will calculate the estimate $\xapx{\sliwin{t}}$ by maximizing the posterior probability of the estimate over the current sliding window, which is $\x{\sliwin{t}}$, conditioned on the prior estimate $\xapxprior{\psliwin{t}}$ with prior information matrix $\infmatapxprior{\psliwin{t}}$ and the current measurement $\y{t}$.
We
apply \eqref{eq_state_transition_model} and \eqref{eq_measurement_model} to yield the approximate objective 
\begin{multline}
    \label{eq_obj_func_approx_norms}
    \Japx(\xapx{\sliwin{t}})  
    = \norm{\xapx{t} - \A{t-1} \xapx{t-1}}^2_{\Q{t-1}^{-1}}
    + \norm{\y{t} - \C{t} \xapx{t}}^2_{\Rinv{t}}
    \\
    + \norm{\xapx{\psliwin{t}} - \xapxprior{\psliwin{t}}}^2_{\infmatapxprior{\psliwin{t}}}.
\end{multline}

We define
\begin{equation}
    \label{eq_Wapx}
    \Wapx{t} = \text{blkdiag}(\Q{t-1}, \Rbf{t}, \infmatapxprior{\psliwin{t}}^{-1}),
\end{equation}
and~$\zapx{t} = [\zero{1\times n} \hspace{8pt} \y{t}^\top \hspace{8pt} \xapxprior{\psliwin{t}}^\top]^\top$. 
Then we can express \eqref{eq_obj_func_approx_norms} 
as
\begin{equation}
    \label{eq_obj_func_approx_quad}
    \Japx(\xapx{\sliwin{t}}) = \frac{1}{2} \xapx{\sliwin{t}}^\top \Kapx{t} \xapx{\sliwin{t}}
    + \fapx{t}^\top \xapx{\sliwin{t}} + \capx{t},
\end{equation}
where $\Kapx{t} = 2\Hbf{t}^\top \Wapx{t}^{-1} \Hbf{t}$, $\fapx{t} = -2\Hbf{t}^\top \Wapx{t}^{-1} \zapx{t}$, and $\capx{t} = \zapx{t}^\top \Wapx{t}^{-1} \zapx{t}$.

\begin{lemma}
    \label{lemma_asynch_analytical_solution}
    For $\Japx$ in \eqref{eq_obj_func_approx_quad},
    the minimizer is the estimate
        $\xapx{\sliwin{t}} = -\Kapx{t}^{-1} \fapx{t}$,
    with information matrix
    \begin{equation}
        \label{eq_infmatapx}
        \infmatapx{\sliwin{t}} = \frac{1}{2} \Kapx{t}.
    \end{equation}
\end{lemma}
\begin{proof}
    \appref{subsect_lemma_asynch_analytical_solution}
\end{proof}


After minimizing $\Japx$ at time $t$, the sliding window is advanced forward by one time step, and the problem is reinitialized at time $t+1$, which requires a new prior distribution $\xapxprior{\psliwin{t+1}}$, which is calculated via
    $\xapxprior{\psliwin{t+1}} = \U{t} \xapx{\sliwin{t}}$.
For computing the matrix~$\infmatapxprior{\psliwin{t+1}}$ we have the following.

\begin{lemma}
    \label{lemma_infmatapxprior_equivalent}
    Let $\infmatapxprior{\psliwin{t}}$ be initialized with $\infmatapxprior{\psliwin{1}} = \infmatapxprior{0:0}$. 
    Then, 
        $\U{t} \infmatest{\sliwin{t}} \U{t}^\top = \U{t} \infmatapx{\sliwin{t}} \U{t}^\top$.
\end{lemma}

\begin{proof}
    \appref{subsect_lemma_infmatapxprior_equivalent}
\end{proof}

Lemma~\ref{lemma_infmatapxprior_equivalent} shows that the update
\begin{equation}
    \label{eq_infmatapxprior}
    \infmatapxprior{\psliwin{t+1}} = \U{t} \infmatapx{\sliwin{t}} \U{t}^\top
\end{equation}
can be used in place of~\eqref{eq_infmatapxprior_vibe}. 
\ap{Intuitively,}
rather than propagating a submatrix of the covariance $\Pest{\sliwin{t}}$ to compute the next prior covariance $\Pestprior{\psliwin{t+1}}$ as in \eqref{eq_cov_prior_update}, 
the update law in \eqref{eq_infmatapxprior} propagates a submatrix of the information matrix $\infmatapx{\sliwin{t}}$ to compute the next prior information matrix $\infmatapxprior{\psliwin{t+1}}$.

This approximate update avoids propagating $\infmatA{t}$ as presented in Proposition~\ref{proposition_dep_synch} to maintain sparsity.
To illustrate this difference, we revisit Example~\ref{example_toy_synch}.
\begin{example}[Example~\ref{example_toy_synch} Revisited]
    \label{example_toy_asynch}
    We consider Example~\ref{example_toy_synch} again,
    with 10 agents where $\T{}=5$, 
    $n_i = m_i = 2$ for all agents $i\in\seq{10}$,
    $\C{t} = \Q{t} = \Rbf{t} = \Pestprior{0:0} = \eye{20 \times 20}$,
    and $\A{t}\in\R{20\times 20}$ is 
    given by \eqref{eq_example_A_bidiagonal}.
    Again, $\A{t}$ is the only model term to present inter-agent dependencies.

    Simulating this example for $t\in\{1,\dots,10\}$ shows that 
    the resulting communication topology is an undirected line graph, such that agent 1 communicates with agent 2, agent 10 communicates with agent 9, and for $i\in\{2,\dots,9\}$, agent $i$ communicates with both agent $i-1$ and agent $i+1$.    
\end{example}


 Example~\ref{example_toy_synch} requires~$O(N^2)$ communication edges, while
 Example~\ref{example_toy_asynch} requires~$O(N)$. 
 We show in Section~\ref{sect_results} that this reduction in communications
 introduces only negligible errors into state estimates, while
 reducing the number of required communication edges by as much as~$96.9$\%.

\section{Asynchronous Algorithm}
\label{sect_asynch_algorithm}

In this section, we 
use the sparse communications induced by the approximate prior 
information matrix 
in \eqref{eq_infmatapxprior},
and we allow agents to compute and communicate asynchronously.




\subsection{Revised Block Coordinate Descent}
We next present an asynchronous 
BCD algorithm by building on the operations presented in Section~\ref{sect_synchronous_algorithm}. 
All $N$ agents collaborate to estimate $\xapx{\sliwin{t}}$ with each agent $i$ estimating $\xapxb{\sliwin{t}}{i}$ and communicating with other agents.
For each $i\in\seq{N}$, agent $i$ stores three quantities onboard:
(i) $\yo{t}{i}$, which is agent $i$'s local copy of the measurements of the target system at time $t$,
(ii) $\xapxoi{\sliwin{t}}{i}{k}$, which is agent $i$'s local copy of the trajectory estimate $\xapx{\sliwin{t}}$ at iteration $k$, and
(iii) $\xapxprioro{\psliwin{t}}{i}$, which is agent $i$'s local copy of the prior $\xapxprior{\psliwin{t}}$\ap{, requiring a total memory size of $n(2T+1)+m$ real numbers.}

We define
\begin{equation} \label{eq_coeffmatxapxprior}
    \coeffmatxapxprior{t} = -2\Pibf{t}^\top \infmatapxprior{\psliwin{t}},
\end{equation}
which allows us to define the dependency sets
\begin{multline}
    \label{eq_depsets_asynch}
    \depsetapxtir{t}{i}{\A{}} = \depop{i}{\priorpropmat{t}}, \quad
    \depsetapxtir{t}{i}{\xapx{}} = \depop{i}{\Kapx{t}}, \\
    \depsetapxtir{t}{i}{\y{}} = \depop{i}{\coeffmaty{t}}, \quad
    \depsetapxtir{t}{i}{\xapxprior{}} = \depop{i}{\coeffmatxapxprior{t}},
\end{multline}
where $\priorpropmat{t}$, $\Kapx{t}$, and $\coeffmaty{t}$ are given by \eqref{eq_alg_prior_propagate}, \eqref{eq_obj_func_approx_quad}, and \eqref{eq_coeffmaty}, respectively.
\ap{By definition, these sets fully describe how inter-robot dependencies in the dynamics and measurement models induce dependencies between agents when executing Algorithm~\ref{alg_bcd_asynch}.}
We now compute the equivalent updates from Section~\ref{sect_synchronous_algorithm} using minimal communications.

Similar to \eqref{eq_alg_prior_propagate}, agent $i$ for $i\in\seq{N}$ initializes the 
estimate $\xapxboi{\sliwin{t}}{i}{i}{0}$ using the prior estimate $\xapxprioro{\psliwin{t}}{i}$ 
according to
\begin{equation}
    \label{eq_xapx_init_minsum}
    \xapxboi{\sliwin{t}}{i}{i}{0} = \sum_{j=\depsetapxtir{t}{i}{\A{}}} \matrixBlock{\priorpropmat{t}}{i}{j} \xapxpriorbo{\psliwin{t}}{j}{i}.
\end{equation}





After initializing their estimates, agents will asynchronously compute updates to $\xapxbo{\sliwin{t}}{i}{i}$ and asynchronously communicate with other agents until a stopping criterion is met.
We define $\updateIterSet{t}{i}$ as the set of all iteration indices for which agent $i$ computes an update to their block.
Since $\updateIterSet{t}{i}$ is a purely analytical tool, agents do not need to know $\updateIterSet{t}{i}$.
When agent $i$ computes an update to $\xapxbo{\sliwin{t}}{i}{i}$ at iteration $k\in\updateIterSet{t}{i}$, the updated quantity is denoted by $\xUpdate{t}{i}{k} \in \R{n_i (\T{t}+1)}$ and calculated with step size $\step$ via
\begin{equation}
    \label{eq_bcd_update_asynch_gradient}
    \xUpdate{t}{i}{k}
    = \xapxboi{\sliwin{t}}{i}{i}{k}
    - \step \cdot \gradb{i} \Japx(\xapxoi{\sliwin{t}}{i}{k}).
\end{equation}
We calculate
$\nabla \Japx(\xapx{\sliwin{t}}) = \Kapx{t} \xapx{\sliwin{t}} + \coeffmaty{t} \y{t} + \coeffmatxapxprior{t} \xapxprior{\psliwin{t}}$,
allowing us to express \eqref{eq_bcd_update_asynch_gradient} as
\begin{multline}
    \label{eq_alg_bcd_update_asynch_minsum}
    \xUpdate{t}{i}{k}
    = \xapxboi{\sliwin{t}}{i}{i}{k}
    - \eta \sum_{j\in\depsetapxtir{t}{i}{\xapx{}}} \matrixBlock{\Kapx{t}}{i}{j} \xapxboi{\sliwin{t}}{j}{i}{k}
    \\
    - \eta \sum_{j\in\depsetapxtir{t}{i}{\y{}}} \matrixBlock{\coeffmaty{t}}{i}{j} \ybo{t}{j}{i}
    - \eta \sum_{j\in\depsetapxtir{t}{i}{\xapxprior{}}} \matrixBlock{\coeffmatxapxprior{t}}{i}{j} \xapxpriorbo{\psliwin{t}}{j}{i}.
\end{multline}


Due to asynchrony, agent $i$ may use outdated information from agent $j$ 
to compute its own block update. 
We define 
$\updateIterRef{t}{j}{i}{k}$ as the
iteration at which agent $j$ originally computed the value of $\xapxbo{\sliwin{t}}{j}{i}$.
That is,
$\xapxboi{\sliwin{t}}{j}{i}{k} = \xapxboi{\sliwin{t}}{j}{j}{\updateIterRef{t}{j}{i}{k}}$.

Using this notation, we can define agent $i$'s update law as 
\begin{align}
    \label{eq_alg_bcd_asynch}
    \xapxboi{\sliwin{t}}{i}{i}{k+1} &= \begin{cases}
        \xUpdate{t}{i}{k} & k \in \updateIterSet{t}{i} 
        \\
        \xapxboi{\sliwin{t}}{i}{i}{k} & k \notin \updateIterSet{t}{i}
    \end{cases}, \nonumber
    \\
    \xapxboi{\sliwin{t}}{j}{i}{k+1} &= \begin{cases}
        \xapxboi{\sliwin{t}}{j}{j}{\updateIterRef{t}{j}{i}{k+1}} & 
        \begin{matrix}
            \text{if } i \text{ receives } \xapxbo{\sliwin{t}}{j}{j} \\ \text{ at iteration } k+1
        \end{matrix}
        \\
        \xapxboi{\sliwin{t}}{j}{i}{k} & \text{otherwise}
    \end{cases}\hspace{-4pt}. \nonumber
\end{align}

Because $\xapxbo{\sliwin{t}}{j}{i}$ is only updated whenever agent $i$ receives new values of $\xapxbo{\sliwin{t}}{j}{j}$ from agent $j$,
old values of $\xapxbo{\sliwin{t}}{j}{i}$ are held constant and potentially reused for many iterations.
\ap{After completing $\maxIter$ iterations,}
agent $i$ calculates their next prior estimate, $\xapxpriorbo{\psliwin{t+1}}{i}{i}$, using their final estimate $\xapxboi{\sliwin{t}}{i}{i}{k}$,
according to
\begin{equation}
    \label{eq_alg_asynch_prior_update}
    \xapxpriorbo{\psliwin{t+1}}{i}{i} = \matrixBlock{\U{t}}{i}{i} \xapxboi{\sliwin{t}}{i}{i}{k}.
\end{equation}

To compute \eqref{eq_xapx_init_minsum} and \eqref{eq_alg_bcd_update_asynch_minsum}, 
agent $i$ must receive $\xapxpriorbo{\psliwin{t}}{j}{i}$ from agents with indices $j\in\depsetapxtir{t}{i}{\A{}} \cup \depsetapxtir{t}{i}{\xapxprior{}}$, the measurement $\ybo{t}{j}{i}$ from agents with indices $j\in\depsetapxtir{t}{i}{\y{}}$, and $\xapxboi{\sliwin{t}}{j}{i}{k}$ from agents with indices $j\in\depsetapxtir{t}{i}{\xapx{}}$.
The next lemma uses the sets
\begin{align}    
    \depsetapxZti{t}{i} &= \begin{cases}
        \depop{i}{\infmatapxprior{0:0}} & \text{if } t\leq T \\
        \begin{pmatrix}
            \depop{i}{\Qinv{t-\T{t}-1}} \cup \\ \depop{i}{\C{t-\T{t}}^\top \Rinv{t-\T{t}} \C{t-\T{t}}}
        \end{pmatrix}
         & \text{if } t>T
    \end{cases} \label{eq_def_depsetapxZti} \\    
    \depsetapxtiq{t}{i}{q} &= \depop{i}{\Qinv{q}} \cup 
    \depop{i}{\Qinv{q} \A{q}} \cup
    \depop{i}{\A{q}^\top \Qinv{q}}
    \\
    &\cup\depop{i}{\A{q}^\top \Qinv{q} \A{q}} \cup
    \depop{i}{\C{q+1}^\top \Rinv{q+1} \C{q+1}}\!. \label{eq_def_depsetapxtiq}
\end{align}

\begin{lemma}
    \label{lemma_depsetapx}
    Given the definitions of $\depsetapxZti{t}{i}$ and $\depsetapxtiq{t}{i}{q}$ from \eqref{eq_def_depsetapxZti} and \eqref{eq_def_depsetapxtiq},
    the dependency sets $\depsetapxtir{t}{i}{\A{}}$, $\depsetapxtir{t}{i}{\xapx{}}$, $\depsetapxtir{t}{i}{\y{}}$, and $\depsetapxtir{t}{i}{\xapxprior{}}$ from \eqref{eq_depsets_asynch} can be expressed using the terms $\A{t}$, $\Q{t}$, $\C{t}$, $\Rbf{t}$, and the initial prior information matrix $\infmatapxprior{0:0}$, via
    $\depsetapxtir{t}{i}{\A{}} = i \cup \depop{i}{\A{t-1}}$,
    $\depsetapxtir{t}{i}{\xapx{}} = \depsetapxZti{t}{i} \cup \bigcup_{q=t-\T{t}}^{t-1} \depsetapxtiq{t}{i}{q}$,
    $\depsetapxtir{t}{i}{\y{}} = \depop{i}{\C{t}^\top \Rinv{t}}$, and
    $\depsetapxtir{t}{i}{\xapxprior{}} = \depsetapxZti{t}{i} \cup \bigcup_{q=t-\T{t}}^{t-2} \depsetapxtiq{t}{i}{q}$.
\end{lemma}

\begin{proof}
    \appref{subsect_lemma_depsetapx}
\end{proof}

This algorithm is summarized in Algorithm~\ref{alg_bcd_asynch}.

\begin{algorithm}
    \caption{Asynchronous BCD}
    \label{alg_bcd_asynch}
    \begin{algorithmic}[1]
        \State $\xapxprioro{0:0}{i} \gets \xapxprior{0:0}$ for all $i\in\seq{N}$
        
        \For{time step $t \in \{1,2,\dots\}$}
            \For{agent $i \in \seq{N}$}
                \Comment{initialization}
                \State Measure $\ybo{t}{i}{i}$
                \State Broadcast $\ybo{t}{i}{i}$ and $\xapxpriorbo{\psliwin{t}}{i}{i}$
                \State Receive $\ybo{t}{j}{i} \gets \ybo{t}{j}{j}$ for all $j \in \depsetapxtir{t}{i}{\y{}}$
                \State Receive $\xapxpriorbo{\psliwin{t}}{j}{i} \gets \xapxpriorbo{\psliwin{t}}{j}{j}$ for all $j \in \depsetapxtir{t}{i}{\A{}}\cup\depsetapxtir{t}{i}{\xapxprior{}}$
                \State Calculate $\xapxboi{\sliwin{t}}{i}{i}{0} \gets$ using~\eqref{eq_xapx_init_minsum} 
            \EndFor
                
            \For{iteration $k \in \{0,\dots, \maxIter-1\}$}
                \Comment{main loop}
                \For{agent $i \in \seq{N}$}
                    \State Broadcast $\xapxboi{\sliwin{t}}{i}{i}{k}$
                    \For{agent $j \in \depsetapxtir{t}{i}{\xapx{}}$}
                        \If{$i$ receives $\xapxbo{\sliwin{t}}{j}{j}$}
                            \State $\xapxboi{\sliwin{t}}{j}{i}{k+1} \gets \xapxboi{\sliwin{t}}{j}{j}{\updateIterRef{t}{j}{i}{k+1}}$
                        \Else
                            \State $\xapxboi{\sliwin{t}}{j}{i}{k+1} \gets \xapxboi{\sliwin{t}}{j}{i}{k}$
                        \EndIf
                    \EndFor
                    \If{$k \in \updateIterSet{t}{i}$}
                        \State $\xapxboi{\sliwin{t}}{i}{i}{k+1} \gets \xUpdate{t}{i}{k}$
                        using~\eqref{eq_alg_bcd_update_asynch_minsum}
                    \Else
                        \State $\xapxboi{\sliwin{t}}{i}{i}{k+1} \gets \xapxboi{\sliwin{t}}{i}{i}{k}$
                    \EndIf
                \EndFor
            \EndFor

            \For{agent $i \in \seq{N}$}
                \Comment{finalization}
                \State Calculate $\xapxpriorbo{\psliwin{t+1}}{i}{i} \gets$ using~\eqref{eq_alg_asynch_prior_update}
            \EndFor
        \EndFor
    \end{algorithmic}
\end{algorithm}

\section{Convergence}
\label{sect_convergence}

As described above, we use a global clock indexed by $k$ to 
track when agents' calculations occur and how communication delays induce outdated information.
The value of~$k$ is used only for analysis and does not need to be known by agents. 
Because agents operate with outdated information, 
we may have $\xapxoi{\sliwin{t}}{i}{k}\neq\xapxoi{\sliwin{t}}{j}{k}$ for all time steps $t$ and iterations $k$.
The following two assumptions are standard in the partially asynchronous optimization literature~\cite{tseng1991rate},
and they will be used to bound agents' disagreements. 

\begin{assumption}
    Let $\updateIterSet{t}{i}$ denote the set of iteration indices for which agent $i$ computes an update to $\xapxbo{\sliwin{t}}{i}{i}$.
    For all iteration indices $k \geq 0$ and for some positive integer $B$,
    $\{k, k+1, \dots, k+B-1\} \cap \updateIterSet{t}{i} \neq \emptyset, 
        \text{ for all } i \in \seq{N}$.
\end{assumption}

This assumption states that agent $i$ will perform a computation to update $\xapxbo{\sliwin{t}}{i}{i}$ at least once every $B$ iterations.

\begin{assumption}
    Let $\updateIterRef{t}{j}{i}{k}$ denote the iteration at which agent $j$ originally computed the value of $\xapxbo{\sliwin{t}}{j}{i}$, 
    such that $\xapxboi{\sliwin{t}}{j}{i}{k} = \xapxboi{\sliwin{t}}{j}{j}{\updateIterRef{t}{j}{i}{k}}$. 
    For some positive integer $B$,
        $0 \leq k - \updateIterRef{t}{j}{i}{k} \leq B - 1.$
    
\end{assumption}

This assumption states that agent $i$'s local copy
of agent $j$'s decision variables, namely
$\xapxbo{\sliwin{t}}{j}{i}$,
is no more than $B$ iterations out of date.
The terms~$\updateIterSet{t}{i}$ 
and~$\cupdateIterRef{t}{j}{i}$
for all~$i$, $j$, and~$t$ are used only for analysis
and
are not known by the agents.


To derive a convergence rate, we first present a few analytical tools.
We define 
\begin{equation}
\xtotal{\sliwin{t}} = \big(\xapxbo{\sliwin{t}}{1}{1}, \xapxbo{\sliwin{t}}{2}{2}, \ldots, \xapxbo{\sliwin{t}}{N}{N}\big)^T \in\R{n(\T{t}+1)}
\end{equation}
as the ``true'' global estimate, defined as the concatenation of agent $i$'s most recent estimate of $\xapxbo{\sliwin{t}}{i}{i}$ for all $i\in\seq{N}$.

We define the sub-optimality between the current estimate $\xtotal{\sliwin{t}}$ and the global minimizer $\x{\sliwin{t}}^*$ as
    $\errorSubopt{t}{k} = \Japx(\xtotali{\sliwin{t}}{k}) - \Japx(\x{\sliwin{t}}^*)$.
As our algorithm converges to the global minimizer, we want $\errorSubopt{t}{k}$
to converge to~$0$.

We now present the primary result of this work. 

\begin{theorem}
    \label{theorem_convergence_rate}
    There exists a scalar $\step_1>0$ such that if $0 < \step < \step_1$, 
    then the sequence of estimates $\{\xtotal{\sliwin{t}}\}$ produced by Algorithm~\ref{alg_bcd_asynch}
    converges at least exponentially to the unique minimizer $\x{\sliwin{t}}^*$ with a $B$-step convergence ratio of $\rho = 1-c\step$, where $c>0$ is some scalar constant.
    That is, for some scalar $a_t>0$, we have
        $\errorSubopt{t}{rB} \leq a_t \rho^{r-1}$.
\end{theorem}

\begin{proof}
    \appref{subsect_theorem_convergence_rate}
\end{proof}

This theorem shows that completing $B$ iterations of Algorithm~\ref{alg_bcd_asynch} reduces the sub-optimality 
$\cerrorSubopt{t}$ by a factor of $\rho$. This result
validates 
our approach for distributed estimation, demonstrating that even under asynchronous computations and communications, 
the true estimate $\xtotal{\sliwin{t}}$ exponentially approaches the global minimizer $\x{\sliwin{t}}^*$.

\section{Results}
\label{sect_results}

We now validate Algorithm~\ref{alg_bcd_asynch} in simulation and on hardware, and
we compare to the DRWT algorithm from~\cite{shorinwa2020distributed}. 

\subsection{Performance Comparison in Simulation} 
\label{sect_results_simulation}
We simulated $N \in \{4, 8, 16, 32, 64, 128\}$ agents estimating the trajectory of $N$ holonomic robots in $\R{2}$. 
The dynamics of the target system are
    $\x{t+1} = \eye{2N}\x{t} + \ubf_t + \procNoise{t}$,
where $\ubf_t = -\mathbf{E}\x{t}$, the matrix
$\mathbf{E}\in\R{2N\times 2N}$ is produced by an LQR controller that regulates the robots to a goal state,
and $\procNoise{t}\sim\gaussian{\zero{}}{\Q{t}}$ with $\Q{t}=\eye{2N\times 2N}$.
The LQR parameters were chosen such that 
the dependency sets $\depsetapxtir{t}{i}{\A{}}, \depsetapxtir{t}{i}{\xapx{}},$ and $\depsetapxtir{t}{i}{\xapxprior{}}$ require agents to communicate with four other agents when computing state estimates. 
The joint measurement model is given by
    $\y{t} = \x{t} + \measNoise{t}$,
where $\measNoise{t}\sim\gaussian{\zero{}}{\Rbf{t}}$ with $\Rbf{t}=\eye{2N\times 2N}$.

Every $B$ iterations of Algorithm~\ref{alg_bcd_asynch}, agents
communicate a single time with just their neighbors in the dependency
sets in Algorithm~\ref{alg_bcd_asynch}. 
Agents executing DRWT~\cite{shorinwa2020distributed} were permitted to perform all-to-all communications.
Algorithm~\ref{alg_bcd_asynch} used the step size $\step=10^{-3}$ and DRWT used the step size $\step=1$.
Agents executing Algorithm~\ref{alg_bcd_asynch} performed $\maxCalcs$ computations per value of~$t$, while agents performing DRWT 
were only able perform an update when communication occurred due to its synchronous algorithm formulation.  

We ran these simulations while varying the number 
of agents $N\in\{4, 8, 16, 32, 64, 128\}$, the number 
of calculations per time step $\maxCalcs\in\{50, 100, 500, 2500, 5000, 10000\}$, and the communication delay $B\in\{1, 2, 5, 10, 100, 500, 2500, 5000\}$. Additionally, each parameter combination was simulated 8 times with different random number generator initializations\footnote{Code available on GitHub: \ap{\url{https://github.com/pooleya19/asynch-dist-estimation-cpp}}}.
For each estimate $\hat{x}$ and MAP estimate $x$, we compute the sub-optimality from MAP as $\norm{\hat{x} - x}_2$.

Figure~\ref{fig_error_vs_num_agents} shows the sub-optimality from MAP for every~$N\in \{4, \ldots, 128\}$ 
with $\maxCalcs=10000$, $B=10$, and averaged over the randomized runs.
We note that for all $N$, the BCD curve visually overlaps with the approximate MAP curve, indicating convergence.
Additionally, the BCD estimate and the approximate MAP estimate attain a small sub-optimality for all $N$.
In contrast, while DRWT approaches the MAP estimate with 4 agents, the DRWT estimation error 
increases quickly as the number of agents increases.
\ap{BCD attains error up to $64$\% lower than DRWT, and the RMSE between the MAP and the approximate MAP is upper bounded by $4.4\cdot 10^{-4}$ across all simulations.}

\begin{figure}
    \centering
    \includegraphics[width=0.9\linewidth]{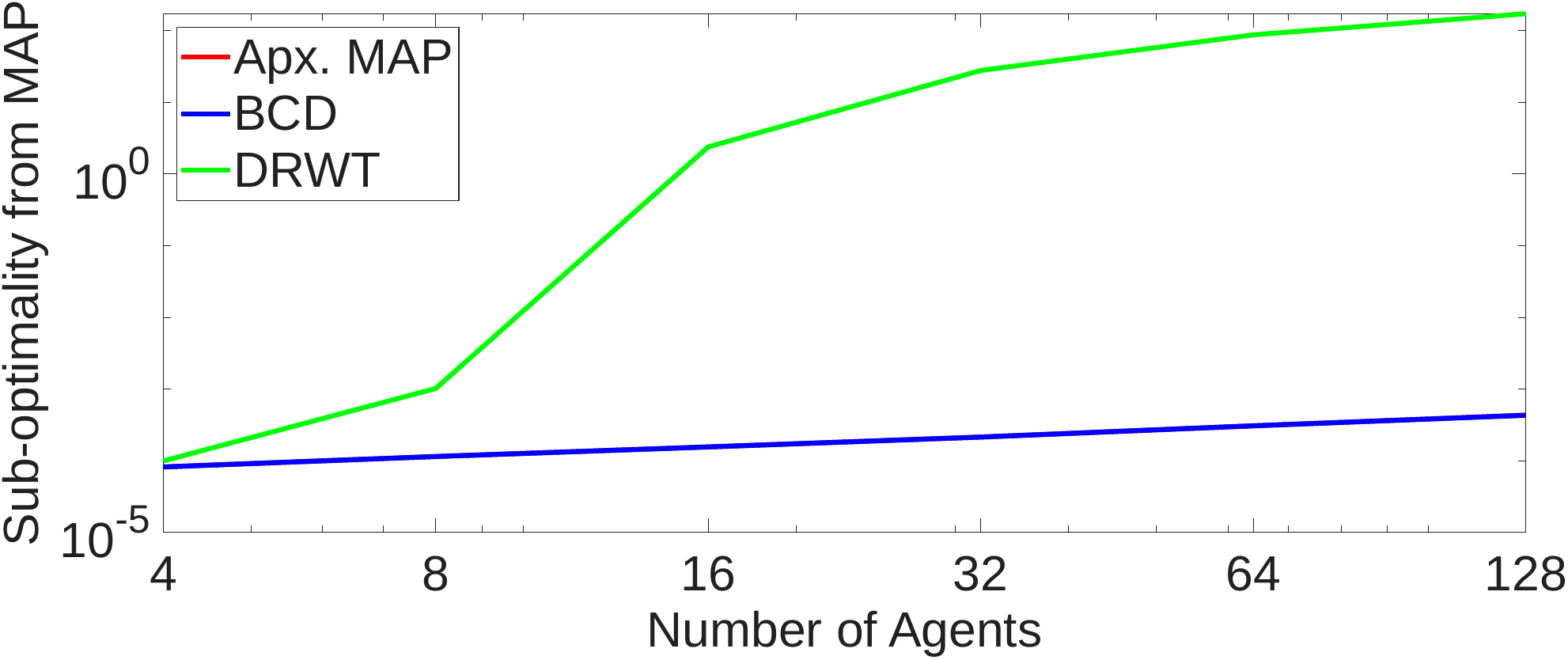}
    \caption{Sub-optimality from the centralized MAP estimate and (i) the approximate MAP estimate (red), (ii) the Algorithm~\ref{alg_bcd_asynch} estimate (blue), and 
    (iii) the DRWT estimate (green), for a range of values of~$N$. We observe that the red and blue curves visually overlap for all numbers of agents.}
    \label{fig_error_vs_num_agents}
\end{figure}


\subsection{Experimental Results}
\label{sect_results_experiments}

We implemented our algorithm with~$4$ agents to estimate the trajectories of~$4$ robots navigating between waypoints.
This demonstration was executed on the Robotarium~\cite{wilson2020robotarium}, shown in Figure~\ref{fig_robotarium}.
Agent $1$ uses GPS measurements of robot $1$, while agent $i$ for $i\in\{2,3,4\}$ measures the relative position between robots $i$ and $i-1$, which induces dependencies between the agents, shown 
as green arrows in Figure~\ref{fig_robotarium}.
The demonstration was repeatedly executed for values of $B\in\{25,50,500,1000,2000,2500,5000\}$ with agents performing $30000$ onboard calculations per time step. 
For each $B$, the agents performed BCD over time steps $t\in\{1,\dots,20\}$ and advanced time steps after $1.32$ seconds.
\ap{All parameters were chosen to demonstrate algorithm performance under varying delay conditions.}

\begin{figure}
    \centering
    \includegraphics[width=0.8\linewidth]{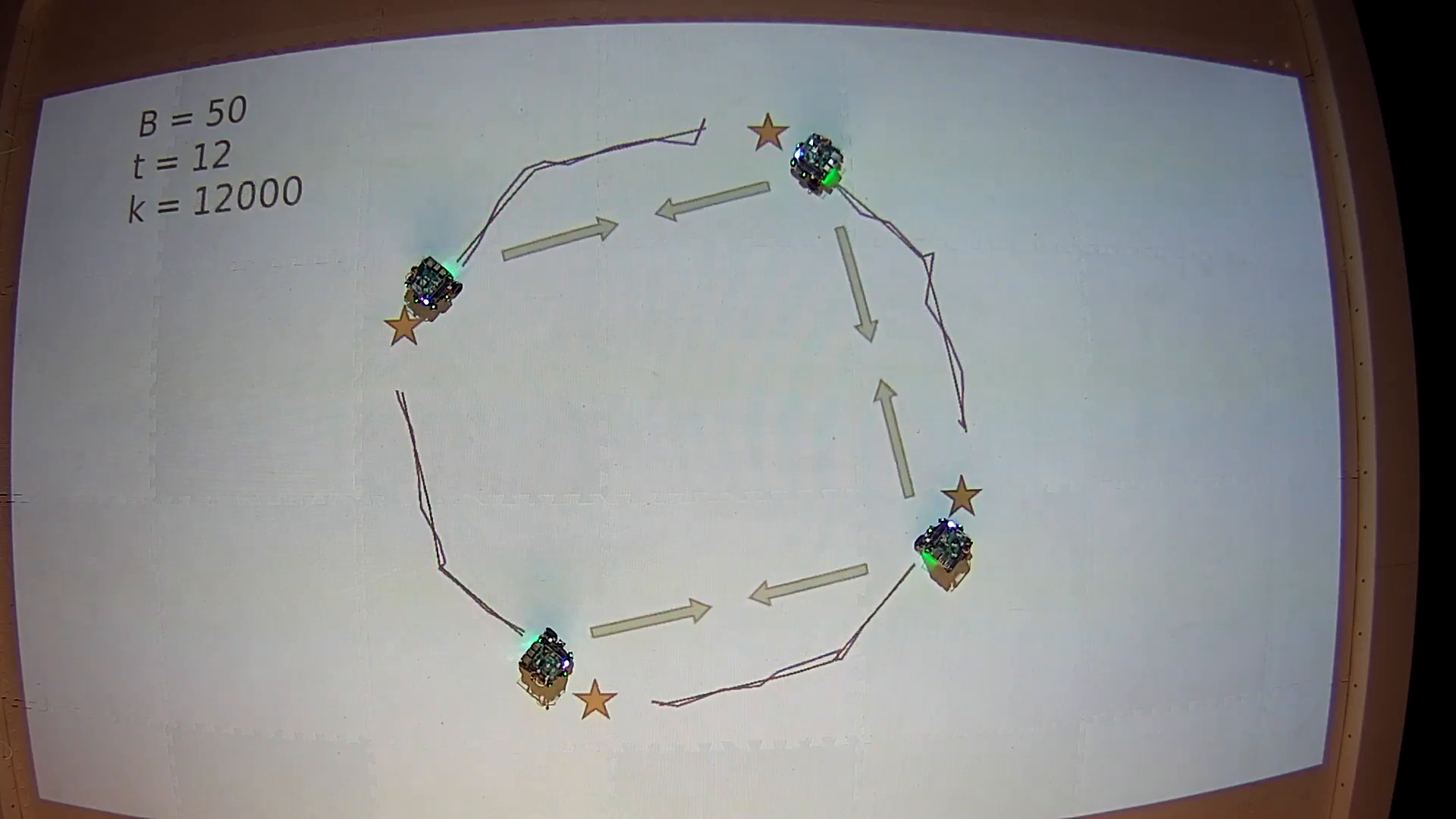}
    \caption{Four robots navigating between waypoints (yellow stars). 
    Black lines show the robots' true trajectories and blue lines show the approximate MAP estimate.
    Red lines show the estimate produced by Algorithm~\ref{alg_bcd_asynch}. 
    }
    \label{fig_robotarium}
\end{figure}

Figure~\ref{fig_demo_convergence_all} shows~$\norm{\hat{x}-x}_2$, where 
$x$ is the approximate MAP estimate and~$\hat{x}$ is the BCD estimate.
When $t$ advances and the problem reinitializes, the error spikes, followed by an exponential decrease as BCD iterations are performed.
As shown in Figure~\ref{fig_demo_convergence_all}, the BCD estimates converge to the approximate MAP solution faster for smaller values of $B$.

\begin{figure}
    \centering
    \includegraphics[width=0.9\linewidth]{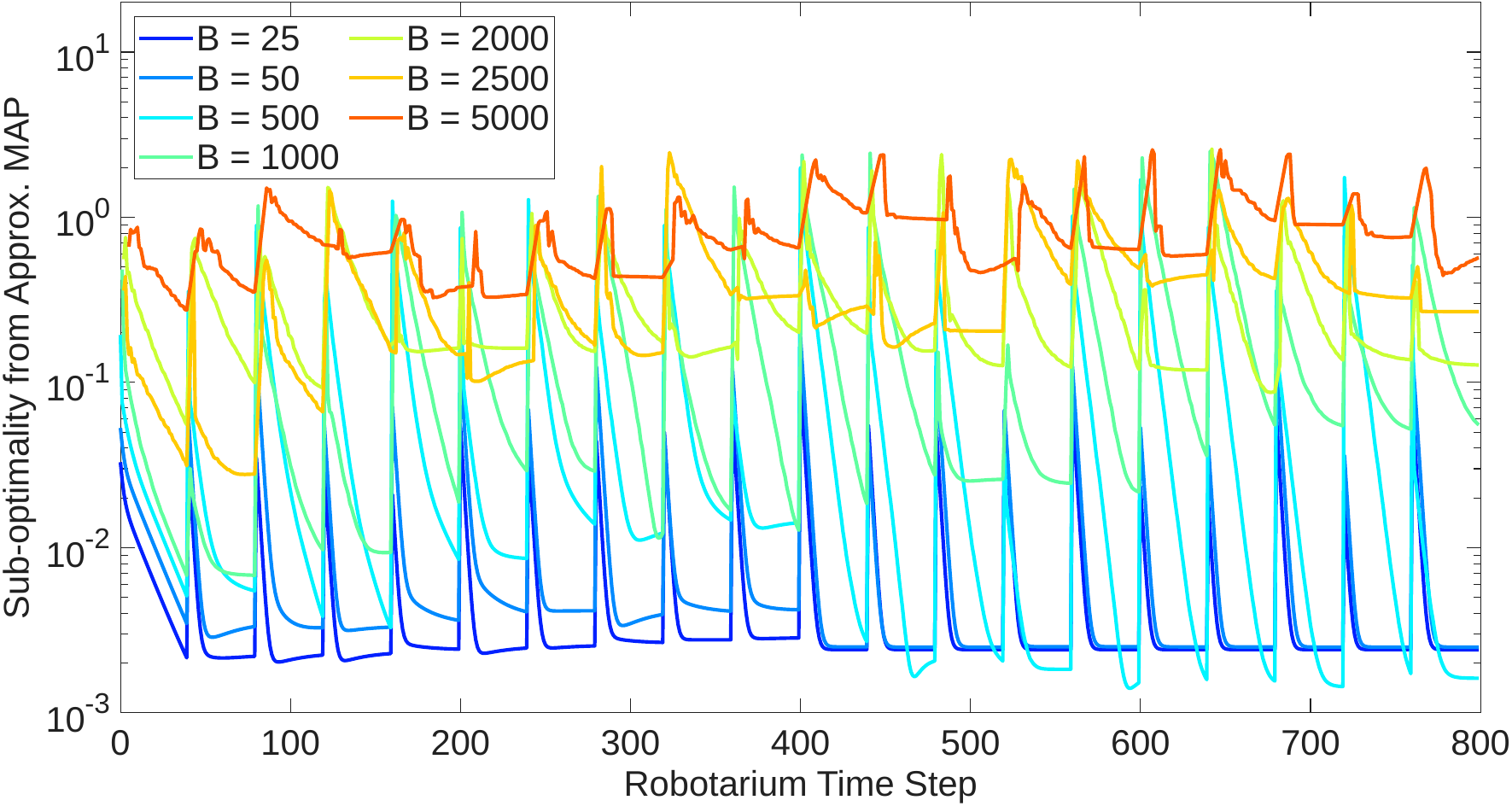}
    \caption{Graph of estimation error between BCD and approximate MAP over time for varying values of $B$.}
    \label{fig_demo_convergence_all}
\end{figure}


\section{Conclusion}
\label{sect_conclusion}

We have presented a novel application of block coordinate descent to estimate a sliding window trajectory under partially asynchronous conditions.
Compared to synchronous sliding window trajectory estimation, our algorithm attains lower error when performed with lower communication rates.
Future work will take advantage of independently selected step sizes for faster convergence.

\bibliographystyle{IEEEtran}
\bibliography{references}

@inproceedings{shorinwa2020distributed,
  title={Distributed multi-target tracking for autonomous vehicle fleets},
  author={Shorinwa, Ola and Yu, Javier and Halsted, Trevor and Koufos, Alex and Schwager, Mac},
  booktitle={2020 IEEE Int. Conf. Robot. Autom. (ICRA)},
  pages={3495--3501},
  year={2020},
  organization={IEEE}
}

@book{bertsekas1989parallel,
  title={Parallel and Distributed Computation: Numerical Methods},
  author={Bertsekas, Dimitri and Tsitsiklis, John},
  year={1989},
  publisher={Prentice-Hall, Inc.}
}

@article{tseng1991rate,
  title={On the rate of convergence of a partially asynchronous gradient projection algorithm},
  author={Tseng, Paul},
  journal={SIAM Journal on Optimization},
  volume={1},
  number={4},
  pages={603--619},
  year={1991},
  publisher={SIAM}
}

@book{thrun2005probabilistic,
    author = {Thrun, Sebastian and Burgard, Wolfram and Fox, Dieter},
    title = {Probabilistic Robotics (Intelligent Robotics and Autonomous Agents)},
    year = {2005},
    isbn = {0262201623},
    publisher = {The MIT Press}
}

@article{testa2025tutorial,
  title={A tutorial on distributed optimization for cooperative robotics: from setups and algorithms to toolboxes and research directions},
  author={Testa, Andrea and Carnevale, Guido and Notarstefano, Giuseppe},
  journal={Proceedings of the IEEE},
  year={2025},
  publisher={IEEE}
}

@article{jaleel2020distributed,
  title={Distributed optimization for robot networks: From real-time convex optimization to game-theoretic self-organization},
  author={Jaleel, Hassan and Shamma, Jeff S},
  journal={Proceedings of the IEEE},
  volume={108},
  number={11},
  pages={1953--1967},
  year={2020},
  publisher={IEEE}
}

@article{shorinwa2024distributedtutorial,
  title={Distributed optimization methods for multi-robot systems: Part 1—a tutorial [tutorial]},
  author={Shorinwa, Ola and Halsted, Trevor and Yu, Javier and Schwager, Mac},
  journal={IEEE Robotics \& Automation Magazine},
  volume={31},
  number={3},
  pages={121--138},
  year={2024},
  publisher={IEEE}
}

@article{khodayi2019distributed,
  title={Distributed state estimation using intermittently connected robot networks},
  author={Khodayi-Mehr, Reza and Kantaros, Yiannis and Zavlanos, Michael M},
  journal={IEEE Transactions on Robotics},
  volume={35},
  number={3},
  pages={709--724},
  year={2019},
  publisher={IEEE}
}

@inproceedings{olfati2007distributed,
  title={Distributed Kalman filtering for sensor networks},
  author={Olfati-Saber, Reza},
  booktitle={2007 46th IEEE Conf. on Decision and Control},
  pages={5492--5498},
  year={2007},
  organization={IEEE}
}

@inproceedings{olfati2005distributed,
  title={Distributed Kalman filter with embedded consensus filters},
  author={Olfati-Saber, Reza},
  booktitle={Proceedings of the 44th IEEE Conf. on Decision and Control},
  pages={8179--8184},
  year={2005},
  organization={IEEE}
}

@article{mateos2010distributed,
  title={Distributed sparse linear regression},
  author={Mateos, Gonzalo and Bazerque, Juan Andr{\'e}s and Giannakis, Georgios B},
  journal={IEEE Transactions on Signal Processing},
  volume={58},
  number={10},
  pages={5262--5276},
  year={2010},
  publisher={IEEE}
}

@article{shi2015extra,
  title={Extra: An exact first-order algorithm for decentralized consensus optimization},
  author={Shi, Wei and Ling, Qing and Wu, Gang and Yin, Wotao},
  journal={SIAM Journal on Optimization},
  volume={25},
  number={2},
  pages={944--966},
  year={2015},
  publisher={SIAM}
}

@article{roumeliotis2002distributed,
  title={Distributed multirobot localization},
  author={Roumeliotis, Stergios I and Bekey, George A},
  journal={IEEE Transactions on Robotics and Automation},
  volume={18},
  number={5},
  pages={781--795},
  year={2002},
  publisher={IEEE}
}

@article{karakus2019redundancy,
  title={Redundancy techniques for straggler mitigation in distributed optimization and learning},
  author={Karakus, Can and Sun, Yifan and Diggavi, Suhas and Yin, Wotao},
  journal={J. Machine Learning Research},
  volume={20},
  number={72},
  pages={1--47},
  year={2019}
}

@article{priyadarshani2025jamming,
  title={Jamming intrusions in extreme bandwidth communication: A comprehensive overview},
  author={Priyadarshani, Richa and Ata, Yal{\c{c}}in and Alouini, Mohamed-Slim and others},
  journal={IEEE Open Journal of the Communications Society},
  year={2025},
  publisher={IEEE}
}

@article{li2014communication,
  title={Communication efficient distributed machine learning with the parameter server},
  author={Li, Mu and Andersen, David G and Smola, Alexander and Yu, Kai},
  journal={Advances in neural information processing systems},
  volume={27},
  year={2014}
}

@article{nassralla2023hybrid,
  title={Hybrid Distributed Optimization for Learning Over Networks With Heterogeneous Agents},
  author={Nassralla, Mohammad H and Akl, Naeem and Dawy, Zaher},
  journal={IEEE Access},
  volume={11},
  pages={103530--103543},
  year={2023},
  publisher={IEEE}
}

@article{wang2025efficient,
  title={Efficient Multi-Cluster Scheduling for Heterogeneous Workloads},
  author={Wang, Xuemin and Li, Xinchi and Xia, Xiaoqing and Yang, Mingchuan},
  journal={IEEE Access},
  volume={13},
  pages={186856--186871},
  year={2025},
  publisher={IEEE}
}

@article{dean2013tail,
  title={The tail at scale},
  author={Dean, Jeffrey and Barroso, Luiz Andr{\'e}},
  journal={Communications of the ACM},
  volume={56},
  number={2},
  pages={74--80},
  year={2013},
  publisher={ACM New York, NY, USA}
}

@article{ziegler2021distributed,
  title={Distributed formation estimation via pairwise distance measurements},
  author={Ziegler, Thomas and Karrer, Marco and Schmuck, Patrik and Chli, Margarita},
  journal={IEEE Robotics and Automation Letters},
  volume={6},
  number={2},
  pages={3017--3024},
  year={2021},
  publisher={IEEE}
}

@article{xu2024d,
  title={$D^{2}$SLAM: Decentralized and distributed collaborative visual-inertial SLAM system for aerial swarm},
  author={Xu, Hao and Liu, Peize and Chen, Xinyi and Shen, Shaojie},
  journal={IEEE Transactions on Robotics},
  volume={40},
  pages={3445--3464},
  year={2024},
  publisher={IEEE}
}

@article{wilson2020robotarium,
  title={The robotarium: Globally impactful opportunities, challenges, and lessons learned in remote-access, distributed control of multirobot systems},
  author={Wilson, Sean and Glotfelter, Paul and Wang, Li and Mayya, Siddharth and Notomista, Gennaro and Mote, Mark and Egerstedt, Magnus},
  journal={IEEE Control Systems Magazine},
  volume={40},
  number={1},
  pages={26--44},
  year={2020},
  publisher={IEEE}
}

@misc{pooley2026technical,
      title={Technical Report: Asynchronous Distributed Trajectory Estimation of Multi-Robot Systems}, 
      author={Adam Pooley and Matthew Hale},
      year={2026},
      eprint={2607.01106},
      archivePrefix={arXiv},
      primaryClass={cs.RO},
      url={https://arxiv.org/abs/2607.01106}, 
}

\appendix



\subsection{Proof of Lemma~\ref{lemma_synch_analytical_solution}}
\label{subsect_lemma_synch_analytical_solution}
\begin{lemma}
    \label{lemma_synch_analytical_solution}
    For $\J$ given by \eqref{eq_obj_func_true_quad},
    the analytical solution to minimizing $\J$ is the estimate
    \begin{equation}
        \label{eq_est_true}
        \xest{\sliwin{t}} = -\K{t}^{-1} \f{t},
    \end{equation}
    with covariance matrix
    \begin{equation}
        \label{eq_cov_true}
        \Pest{\sliwin{t}} = 2\K{t}^{-1}.
    \end{equation}
\end{lemma}
\begin{proof}
    For $\J$ as given in \eqref{eq_obj_func_true_quad},
    the gradient $\nabla \J$ is
    \begin{equation}
        \label{eq_gradient_true_lemma}
        \nabla \J(\xest{\sliwin{t}}) = \K{t} \xest{\sliwin{t}} + \f{t},
    \end{equation}
    and the Hessian $\nabla^2 \J(\xest{\sliwin{t}})$ is
    \begin{equation}
        \label{eq_hessian_true}
        \nabla^2 \J(\xest{\sliwin{t}}) = \K{t}.
    \end{equation}

    We briefly prove that the Hessian $\nabla^2 \J(\xest{\sliwin{t}}) = \K{t}$ is positive definite, i.e., $\K{t} \succ 0$.
    Recall that we assume $\Q{t}, \Rbf{t}, \Pestprior{\psliwin{t}} \succ 0$ for all $t$.
    We know $\W{t}\succ 0$ by \eqref{eq_def_W}.
    Because $\Hbf{t}$ is full column-rank by \eqref{eq_def_H}, its null space is the zero vector by the rank-nullity theorem.
    As such, from \eqref{eq_obj_func_true_quad}, we know $\K{t} = 2\Hbf{t}^\top \W{t}^{-1} \Hbf{t} \succ 0$.
    Additionally, because matrix inversion preserves positive definiteness, $\Pest{\sliwin{t}}\succ 0$ from \eqref{eq_cov_true},
    and because taking a principle submatrix preserves positive definiteness from the Cauchy interlacing theorem, $\Pestprior{t-\T{t+1}+1:t}\succ 0$ from \eqref{eq_cov_prior_update}.

    Because $\J$ is quadratic with a positive definite Hessian $\K{t} \succ 0$,
    it is strongly convex, and its unique minimizer
    can be found by setting the gradient $\nabla J(\xest{\sliwin{t}})$ to zero and solving for $\xest{\sliwin{t}}$, which gives
    \begin{equation}
        \label{eq_est_true_lemma}
        \xest{\sliwin{t}} = -\K{t}^{-1} \f{t}.
    \end{equation}
    \ap{We calculate its covariance according to
    \begin{align}
        \Pest{\sliwin{t}} &= \E{\x{\sliwin{t}} - \E{\x{\sliwin{t}}}} \E{\x{\sliwin{t}} - \E{\x{\sliwin{t}}}}^\top \\
                          &= \E{\x{\sliwin{t}} - \xest{\sliwin{t}}} \E{\x{\sliwin{t}} - \xest{\sliwin{t}}}^\top.
    \end{align}
    Substituting in $\xest{\sliwin{t}}$ from \eqref{eq_est_true_lemma}} and simplifying yields
    \begin{equation}
        \Pest{\sliwin{t}} = 2\K{t}^{-1}.
    \end{equation}
    
\end{proof}

\subsection{Proof of Proposition~\ref{proposition_dep_synch}}
\label{subsect_proposition_dep_synch}
\begin{proof}
    We begin with the definition of $\depsettir{t}{i}{\xest{}}$ as defined in \eqref{eq_depsets_synch}.
    By using \eqref{eq_cov_true}, we say $\depsettir{t}{i}{\xest{}} = \depop{i}{\Pest{\sliwin{t}}^{-1}}$.

    We define
    $\infmatestprior{0:0} = \Pestprior{0:0}^{-1}$,
    $\infmatest{\sliwin{t}} = \Pest{\sliwin{t}}^{-1}$, and
    $\infmatestprior{\psliwin{t}} = \Pestprior{\psliwin{t}}^{-1}$,
    such that 
    $\depsettir{t}{i}{\xest{}} = \depop{i}{\infmatest{\sliwin{t}}}$.
    
    We unravel the recursive update structure to express $\infmatest{\sliwin{t}}$ in closed form for $t\in\seq{T}$.
    
    From $\K{t}$ as defined in \eqref{eq_obj_func_true_quad} and \eqref{eq_cov_true}, we can calculate the information matrix $\infmatest{\sliwin{t}}$ associated with estimate $\xest{\sliwin{t}}$ according to
    \begin{equation}
        \infmatest{\sliwin{t}} = \Hbf{t}^\top \W{t}^{-1} \Hbf{t},
    \end{equation}
    \ap{where $\Hbf{t}$ and $\W{t}$ are from \eqref{eq_def_H} and \eqref{eq_def_W}.}
    Expanding the right hand side yields
    \begin{equation}
        \label{eq_infmatest_sum}
        \infmatest{\sliwin{t}} = \F{t}^\top \Q{t-1}^{-1} \F{t} + \G{t}^\top \Rinv{t} \G{t} + \Pibf{t}^\top \infmatestprior{\psliwin{t}} \Pibf{t}.
    \end{equation}
    
    The nonzero terms from $\F{t}^\top \Q{t-1}^{-1} \F{t} + \G{t}^\top \Rinv{t} \G{t}$ contribute to a principal submatrix \ap{of $\infmatest{\sliwin{t}}$}, denoted $\infsub{\ell}\in\R{2n\times 2n}$, which we define by
    \begin{equation}
        \label{eq_def_infsub}
        \infsub{\ell} = \begin{bmatrix}
            \A{\ell-1}^\top \Q{\ell-1}^{-1} \A{\ell-1} & -\A{\ell-1}^\top \Q{\ell-1}^{-1} \\
            -\Q{\ell-1}^{-1} \A{\ell-1} & \Q{\ell-1}^{-1} + \C{\ell}^\top \Rinv{\ell} \C{\ell}
        \end{bmatrix},
    \end{equation}
    such that
    \begin{multline}
        \F{t}^\top \Q{t-1}^{-1} \F{t} + \G{t}^\top \Rinv{t} \G{t} \\
        = \begin{bmatrix}
            \zero{n(\T{t}-1)\times n(\T{t}-1)} & \zero{n(\T{t}-1)\times 2n} \\
            \zero{2n \times n(\T{t}-1)} & \infsub{t}
        \end{bmatrix}.
    \end{multline}
    To rigorously handle where $\infsub{t}$ appears, we define $\infsubpos{a}{b}\in\R{2n\times bn}$ for integers $a,b$ according to
    \begin{equation}
        \infsubpos{a}{b} = \begin{bmatrix}
            \zero{2n\times n(a-1)} & \eye{2n\times 2n} & \zero{2n\times n(b-a-1)}
        \end{bmatrix},
    \end{equation}
    allowing us to express
    \begin{equation}
        \label{eq_infsub_infsubpos}
        \F{t}^\top \Q{t-1}^{-1} \F{t} + \G{t}^\top \Rinv{t} \G{t}
        = \infsubpos{\T{t}}{\T{t}+1}^\top \infsub{t} \infsubpos{\T{t}}{\T{t}+1}.
    \end{equation}
    Plugging \eqref{eq_infsub_infsubpos} into \eqref{eq_infmatest_sum} yields
    \begin{equation}
        \label{eq_infmatest_infsub}
        \infmatest{\sliwin{t}} = \infsubpos{\T{t}}{\T{t}+1}^\top \infsub{t} \infsubpos{\T{t}}{\T{t}+1} + \Pibf{t}^\top \infmatestprior{\psliwin{t}} \Pibf{t}.
    \end{equation}
    
    The next prior information matrix $\infmatestprior{\psliwin{t+1}}$ can be calculated by expressing the prior covariance calculation in \eqref{eq_cov_prior_update} using information matrices, yielding
    \begin{equation}
        \label{eq_infmatestprior_update}
        \infmatestprior{\psliwin{t+1}} = (\U{t} \infmatest{\sliwin{t}}^{-1} \U{t}^\top)^{-1}.
    \end{equation}
    When $t\in\seq{T-1}$, we know \ap{\eqref{eq_def_U} reduces to} $\U{t} = \eye{n\T{t+1}\times n\T{t+1}}$,
    meaning that \eqref{eq_infmatestprior_update} reduces to
    \begin{equation}
        \label{eq_infmatestprior_early}
        \infmatestprior{\psliwin{t+1}} = \infmatest{\sliwin{t}} \quad \text{for } t \in \seq{T-1}.
    \end{equation}
    
    We define $\infpriorpos{a}\in\R{n\times an}$ for some integer $a$ according to
    \begin{equation}
        \infpriorpos{a} = \begin{bmatrix}
            \eye{n\times n} & \zero{n\times n(a-1)}
        \end{bmatrix}.
    \end{equation}
    
    The recursive calculation of $\infmatest{\sliwin{t}}$ for $t\in \seq{T}$ can be expressed by using the initial condition $\infmatestprior{\psliwin{1}} = \infmatestprior{0:0}$ and by combining \eqref{eq_infmatest_infsub} with \eqref{eq_infmatestprior_early} which yields
    \begin{equation}
        \label{eq_infmatest_early_sum}
        \infmatest{\sliwin{t}} = \infpriorpos{\T{t}+1}^\top \infmatestprior{0:0} \infpriorpos{\T{t}+1} + \sum_{\ell=1}^{\T{t}} \infsubpos{\ell}{\T{t}+1}^\top \infsub{t-\T{t}+\ell} \infsubpos{\ell}{\T{t}+1}.
    \end{equation}
    
    

    Next, we will derive the updates for $\infmatest{\sliwin{t}}$ for $t > T$ by utilizing the Schur complement.
    For $t\geq T$, we know \ap{\eqref{eq_def_U} reduces to} $\U{t} = \begin{bmatrix}
        \zero{nT \times n} & \eye{nT \times nT}
    \end{bmatrix}$, allowing us to express \eqref{eq_infmatestprior_update} as
    \begin{equation}
        \label{eq_infmatestprior_big}
        \infmatestprior{\psliwin{t+1}} = \left(
        [\zero{nT\times n} \hspace{5pt} \eye{nT\times nT}]
        \infmatest{\sliwin{t}}^{-1} \begin{bmatrix}
            \zero{n\times nT} \\ \eye{nT\times nT}
        \end{bmatrix} \right)^{-1}.
    \end{equation}
    \ap{Using \eqref{eq_infmatestprior_big}}
    is equivalent to inverting $\infmatest{\sliwin{t}}$, taking the lower right $nT\times nT$ principal submatrix, then inverting again.
    We partition
    \begin{equation}
        \label{eq_infmatest_partition}
        \infmatest{\sliwin{t}} = \begin{bmatrix}
            \infmatA{t} & \infmatB{t} \\
            \infmatC{t} & \infmatD{t}
        \end{bmatrix} \quad \text{for } t\geq T,
    \end{equation}
    such that $\infmatA{t}\in\R{n\times n}, \infmatB{t} = \infmatC{t}^\top \in \R{n\times nT}$, and $\infmatD{t}\in\R{nT\times nT}$.
    
    We substitute the partitioning of $\infmatest{\sliwin{t}}$ from \eqref{eq_infmatest_partition} into \eqref{eq_infmatestprior_big}, which yields
    \begin{equation}
        \infmatestprior{\psliwin{t+1}} = \left( 
        [\zero{nT\times n} \hspace{5pt} \eye{nT\times nT}] 
        \begin{bmatrix}
            \infmatA{t} & \hspace{-5pt}\infmatB{t} \\
            \infmatC{t} & \hspace{-5pt}\infmatD{t}
        \end{bmatrix}^{-1} \begin{bmatrix}
            \zero{n\times nT} \\ \eye{nT\times nT}
        \end{bmatrix} \right)^{-1}.
    \end{equation}
    Applying the Schur complement yields
    \begin{equation}
        \label{eq_infmatestprior_schur}
        \infmatestprior{\psliwin{t+1}} = \infmatD{t} - \infmatC{t} \infmatA{t}^{-1} \infmatB{t},
        \quad \text{for } t \geq T.
    \end{equation}

    We evaluate the partitions of $\infmatest{\sliwin{t}}$ at $t=T$ using \eqref{eq_infmatest_early_sum}, yielding
    \begin{align}
        \infmatA{t} &= \infmatestprior{0:0} + \A{0}^\top \Q{0}^{-1} \A{0}
        \\
        \infmatB{t} 
        &= \infmatC{t}^\top = -\A{0}^\top \Q{0}^{-1} \infpriorpos{T},
    \end{align}
    and
    \begin{multline}
        \infmatD{t} = \infpriorpos{\T{}}^\top (\Q{0}^{-1} + \C{1}^\top \Rinv{1} \C{1}) \infpriorpos{\T{}}
        \\
        + \sum_{\ell=1}^{\T{}-1} \infsubpos{\ell}{\T{}}^\top \infsub{t-T+\ell+1} \infsubpos{\ell}{\T{}}.
    \end{multline}

    \ap{T}he structure of $\infmatB{t}$ and $\infmatC{t}$ cause\ap{s} $\infmatA{t}^{-1}$ to propagate to the upper-left $n\times n$ submatrix of the resulting matrix \ap{$\infmatestprior{\psliwin{t+1}}$ because}
    \begin{equation}
        \label{eq_schurproduct}
         \infmatC{t} \infmatA{t}^{-1} \infmatB{t} = \infpriorpos{T}^\top \Q{t-\T{}}^{-1} \A{t-\T{}} \infmatA{t}^{-1} \A{t-\T{}}^\top \Q{t-\T{}}^{-1} \infpriorpos{T}.
    \end{equation}

    To help us discuss the
    quantity that recursively propagates, we define
    \begin{multline}
        \label{eq_def_infmatcom}
        \infmatcom{t+1} = \Q{t-\T{}}^{-1} + \C{t-\T{}+1}^\top \Rinv{t-\T{}+1} \C{t-\T{}+1}
        \\
        - \Q{t-\T{}}^{-1} \A{t-\T{}} \infmatA{t}^{-1} \A{t-\T{}}^\top \Q{t-\T{}}^{-1}.
    \end{multline}
    We calculate $\infmatestprior{\psliwin{t+1}}$ according to \eqref{eq_infmatestprior_schur} which yields
    \begin{equation}
        \label{eq_def_infmatprior_later}
        \infmatestprior{\psliwin{t+1}} = \infpriorpos{T}^\top \infmatcom{t+1} \infpriorpos{T} + \sum_{\ell=1}^{\T{}-1} \infsubpos{\ell}{\T{}}^\top \infsub{t-T+\ell+1} \infsubpos{\ell}{\T{}}.
    \end{equation}
    Now, we increment $t = T + 1$ and calculate $\infmatest{\sliwin{t}}$ from \eqref{eq_infmatest_infsub}
    using $\infmatestprior{\psliwin{t}}$ from \eqref{eq_def_infmatprior_later} which yields
    \begin{equation}
        \label{eq_infmatest_one_iteration}
        \infmatest{\sliwin{t}} = \infpriorpos{T+1}^\top \infmatcom{t} \infpriorpos{T+1} 
        + \sum_{\ell=1}^T \infsubpos{\ell}{T+1}^\top \infsub{t-T+\ell} \infsubpos{\ell}{T+1}.
    \end{equation}

    To consider the effects of recursively propagating $\infmatcom{t}$, we now consider the partitions of $\infmatest{\sliwin{t}}$ from \eqref{eq_infmatest_one_iteration}, yielding
    \begin{align}
        \label{eq_def_infmatA}
        \infmatA{t} &= \infmatcom{t} + \A{t-T}^\top \Qinv{t-T} \A{t-T}
        \\
        \infmatB{t} &= \infmatC{t}^\top = -\A{t-T}^\top \Qinv{t-T} \infpriorpos{T},
    \end{align}
    and
    \begin{multline}
        \label{eq_def_infmatD}
        \infmatD{t} = \infpriorpos{T}^\top (\Qinv{t-T} + \C{t-T+1}^\top \Rinv{t-T+1} \C{t-T+1}) \infpriorpos{T}
        \\ + \sum_{\ell=1}^{T-1} \infsubpos{\ell}{T}^\top \infsub{t-T+\ell+1} \infsubpos{\ell}{T}.
    \end{multline}

    By contrasting the partitions of $\infmatest{\sliwin{t}}$ evaluated at $t=T$ and $t=T+1$, we can see that the recursive updates of $\infmatestprior{\psliwin{t}}$ and $\infmatest{\sliwin{t}}$ from \eqref{eq_infmatestprior_schur} and \eqref{eq_infmatest_infsub} do not change the structure of $\infmatB{t}, \infmatC{t}$, or $\infmatD{t}$.
    This means that when evaluating $\infmatestprior{\psliwin{t+1}}$ using the Schur complement in \eqref{eq_infmatestprior_schur}, $\infmatC{t}^\top \infmatA{t}^{-1} \infmatB{t}$ can always be expressed as in \eqref{eq_schurproduct}, which means $\infmatcom{t+1}$ can always be expressed as in \eqref{eq_def_infmatcom}.
    By combining \eqref{eq_def_infmatcom} and \eqref{eq_def_infmatA}, we derive the recursive update of $\infmatA{t}$, given by
    \begin{multline}
        \infmatA{t} = 
        \C{t-T}^\top \Rinv{t-T} \C{t-T} + \A{t-T}^\top \Qinv{t-T} \A{t-T}
        \\
        + \Qinv{t-T-1} - \Qinv{t-T-1} \A{t-T-1} \infmatA{t-1}^{-1} \A{t-T-1}^\top \Qinv{t-T-1}.
    \end{multline}

    Formally, because $\infmatA{t}$ is a submatrix of $\infmatest{\sliwin{t}}$, we know $\depop{i}{\infmatA{t}} \subseteq \depop{i}{\infmatest{\sliwin{t}}}$.
    Furthermore, because $\depsettir{t}{i}{\xest{}} = \depop{i}{\infmatest{\sliwin{t}}}$, we know $\depop{i}{\infmatA{t}} \subseteq \depsettir{t}{i}{\xest{}}$.

\end{proof}

\subsection{Proof of Lemma~\ref{lemma_asynch_analytical_solution}}
\label{subsect_lemma_asynch_analytical_solution}
\begin{proof}
    \ap{
    For $\Japx$ as given in \eqref{eq_obj_func_approx_quad}, the gradient $\nabla \Japx$ is calculated as
    \begin{equation}
        \nabla \Japx(\xapx{\sliwin{t}}) = \Kapx{t} \xapx{\sliwin{t}} + \fapx{t},
    \end{equation}
    and the Hessian $\nabla^2 \Japx(\xapx{\sliwin{t}})$ is calculated as
    \begin{equation}
        \label{eq_hessian_approx}
        \nabla^2 \Japx(\xapx{\sliwin{t}}) = \Kapx{t}.
    \end{equation}}

    \ap{To assist our proof, we will prove by induction that $\infmatapxprior{\psliwin{t}} \succ 0$ for all $t \geq 1$.
    Recall that we have assumed $\infmatapxprior{0:0}, \Q{t}, \Rbf{t} \succ 0$ for all $t \geq 0$.
    As such, our base case is proven by assumption, because $\infmatapxprior{\psliwin{t}} = \infmatapxprior{0:0} \succ 0$ for $t = 1$.}

    \ap{Now, we assume $\infmatapxprior{\psliwin{t}} \succ 0$ for some $t\geq 1$ and we seek to prove that this implies $\infmatapxprior{\psliwin{t+1}} \succ 0$.
    Because $\Qinv{t-1}, \Rinv{t}, \infmatapxprior{\psliwin{t}} \succ 0$, we know
    $\Wapxinv{t} \succ 0$ from \eqref{eq_Wapx}.
    Because $\Hbf{t}$ is full column-rank from \eqref{eq_def_H}, its null space is the zero vector by the rank-nullity theorem.
    As such, from \eqref{eq_obj_func_approx_quad} we know that $\Kapx{t} = 2\Hbf{t}^\top \Wapx{t}^{-1} \Hbf{t} \succ 0$.}

    Because $\Kapx{t} \succ 0$, then $\Japx(\xapx{\sliwin{t}})$ is quadratic with a positive definite Hessian $\Kapx{t} \succ 0$.
    As such, the minimizer can be found by setting the gradient $\nabla \Japx(\xapx{\sliwin{t}})$ to zero and solving for $\xapx{\sliwin{t}}$, which gives
    \begin{equation}
        \label{eq_xapx_lemma}
        \xapx{\sliwin{t}} = -\Kapx{t}^{-1} \fapx{t}.
    \end{equation}
    We calculate the covariance according to
    \begin{align}
        \Papx{\sliwin{t}} 
        &= \E{\x{\sliwin{t}} - \E{\x{\sliwin{t}}}} \E{\x{\sliwin{t}} - \E{\x{\sliwin{t}}}}^\top
        \\
        &= \E{\x{\sliwin{t}} - \xapx{\sliwin{t}}} \E{\x{\sliwin{t}} - \xapx{\sliwin{t}}}^\top.
    \end{align}
    Substituting in $\xapx{\sliwin{t}}$ from \eqref{eq_xapx_lemma} and simplifying yields
    \begin{equation}
        \label{eq_cov_approx}
        \Papx{\sliwin{t}} = 2\Kapx{t}^{-1}.
    \end{equation}
    By inverting \eqref{eq_cov_approx}, we compute the information matrix $\infmatapx{\sliwin{t}}$ via
    \begin{equation}
        \label{eq_infmatapx_lemma}
        \infmatapx{\sliwin{t}} = \frac{1}{2} \Kapx{t}.
    \end{equation}

    \ap{Combining $\Kapx{t}\succ 0$ and \eqref{eq_infmatapx_lemma} implies $\infmatapx{\sliwin{t}} \succ 0$.
    We next consider the recursive update of $\infmatapxprior{\psliwin{t+1}}$ from $\infmatapx{\sliwin{t}}$ given by \eqref{eq_infmatapxprior}.}

    If $t<\Tmax$, then \eqref{eq_infmatapxprior} reduces to $\infmatapxprior{\psliwin{t+1}} = \infmatapx{\sliwin{t}}$, which directly implies $\infmatapxprior{\psliwin{t+1}} \succ 0$.
    If $t\geq\Tmax$, then \eqref{eq_infmatapxprior} reduces to
    $\infmatapxprior{\psliwin{t+1}} = \begin{bmatrix}
        \zero{n\Tmax\times n} & \eye{n\Tmax\times n\Tmax}
    \end{bmatrix}\infmatapx{\sliwin{t}}\begin{bmatrix}
        \zero{n\times n\Tmax} \\ \eye{n\Tmax\times n\Tmax}
    \end{bmatrix}$.
    The Cauchy interlacing theorem tells us that the eigenvalues of $\infmatapxprior{\psliwin{t+1}}$ are lowered bounded by the eigenvalues of $\infmatapx{\sliwin{t}}$, which implies $\infmatapxprior{\psliwin{t+1}} \succ 0$.
    As such, we have proven by induction that $\infmatapxprior{\psliwin{t}}\succ0 \implies \infmatapxprior{\psliwin{t+1}}\succ0$ for all $t\geq 1$.

    \ap{As such, for all $t\geq 1$, we have proven $\Kapx{t}\succ 0$, which implies that \eqref{eq_xapx_lemma} and \eqref{eq_infmatapx_lemma} hold for all $t \geq 1$.}
\end{proof}

\subsection{Proof of Lemma~\ref{lemma_infmatapxprior_equivalent}}
\label{subsect_lemma_infmatapxprior_equivalent}
\begin{proof}
    To assist our analysis, for some integers $\ell, a$, and $b$ we define
    \begin{align}
        \infsub{\ell} &= \begin{bmatrix}
            \A{\ell-1}^\top \Q{\ell-1}^{-1} \A{\ell-1} & -\A{\ell-1}^\top \Q{\ell-1}^{-1} \\
            -\Q{\ell-1}^{-1} \A{\ell-1} & \Q{\ell-1}^{-1} + \C{\ell}^\top \Rinv{\ell} \C{\ell}
        \end{bmatrix}
        \\
        \infsubpos{a}{b} &= \begin{bmatrix}
            \zero{2n\times n(a-1)} & \eye{2n\times 2n} & \zero{2n\times n(b-a-1)}
        \end{bmatrix}
        \\
        \infpriorpos{a} &= \begin{bmatrix}
            \eye{n\times n} & \zero{n\times n(a-1)}
        \end{bmatrix}.
    \end{align}

    From Appendix~\ref{subsect_proposition_dep_synch} we have $\infmatest{\sliwin{t}}$ for $t\in\seq{T}$ from \eqref{eq_infmatest_early_sum}.
    Additionally, for the partitioning of $\infmatest{\sliwin{t}}$ as described in \eqref{eq_infmatest_partition}, we have $\infmatD{t}$ for $t\geq T$ from \eqref{eq_def_infmatD}.
    Combining these yields
    \begin{multline}
        \label{eq_UinfmatestU_early}
        \U{t} \infmatest{\sliwin{t}} \U{t}^\top = \infpriorpos{\T{t+1}}^\top \infmatestprior{0:0} \infpriorpos{\T{t+1}} 
        \\+ \sum_{\ell=1}^{\T{t}} \infsubpos{\ell}{\T{t+1}}^\top \infsub{t-\T{t}+\ell} \infsubpos{\ell}{\T{t+1}},
    \end{multline}
    for $t\in\seq{T-1}$, and
    \begin{multline}
        \label{eq_UinfmatestU_later}
        \U{t} \infmatest{\sliwin{t}} \U{t}^\top = \infpriorpos{T}^\top (\Qinv{t-T} + \C{t-T+1}^\top \Rinv{t-T+1} \C{t-T+1}) \infpriorpos{T}
        \\ + \sum_{\ell=1}^{T-1} \infsubpos{\ell}{T}^\top \infsub{t-T+\ell+1} \infsubpos{\ell}{T},
    \end{multline}
    for $t\geq T$.
    
    We now seek to express $\U{t} \infmatapx{\sliwin{t}} \U{t}^\top$ in closed form.
    We begin by initializing $\infmatapxprior{\psliwin{t}}$ with $\infmatapxprior{\psliwin{1}} = \infmatapxprior{0:0}$.

    We begin by writing $\infmatapx{\sliwin{t}}$ using \eqref{eq_infmatapx} and \eqref{eq_obj_func_approx_quad} according to
    \begin{multline}
        \infmatapx{\sliwin{t}} = \Hbf{t}^\top \Wapx{t}^{-1} \Hbf{t}
        \\
        = \F{t}^\top \Qinv{t-1} \F{t} + \G{t}^\top \Rinv{t} \G{t} + \Pibf{t}^\top \infmatapxprior{\psliwin{t}} \Pibf{t}.
    \end{multline}

    Substituting in $\infsub{t}$ yields
    \begin{equation}
        \label{infmatapx_simple_sum}
        \infmatapx{\sliwin{t}} = \infsubpos{\T{t}}{\T{t}+1}^\top \infsub{t} \infsubpos{\T{t}}{\T{t}+1} + \Pibf{t}^\top \infmatapxprior{\psliwin{t}} \Pibf{t}.
    \end{equation}

    For $t\in\seq{T-1}$, we know \eqref{eq_def_U} reduces to $\U{t} = \eye{n\T{t+1}\times n\T{t+1}}$, meaning that \eqref{eq_infmatapxprior} reduces to
    \begin{equation}
        \label{infmatapxprior_early}
        \infmatapxprior{\psliwin{t+1}} = \infmatapx{\sliwin{t}} \quad \text{for } t\in\seq{T-1}.
    \end{equation}

    By combining \eqref{infmatapx_simple_sum} and \eqref{infmatapxprior_early} and the initial condition $\infmatapxprior{\psliwin{1}} = \infmatapxprior{0:0}$, we can express $\infmatapx{\sliwin{t}}$ for $t\in\seq{T}$ via
    \begin{equation}
        \label{eq_infmatapx_early}
        \infmatapx{\sliwin{t}} = \infpriorpos{\T{t}+1}^\top \infmatapxprior{0:0} \infpriorpos{\T{t}+1}
        + \sum_{\ell=1}^{\T{t}} \infsubpos{\ell}{\T{t}+1}^\top \infsub{t-\T{t}+\ell} \infsubpos{\ell}{\T{t}+1}.
    \end{equation}

    For $t\geq T$, we know \eqref{eq_def_U} reduces to $\U{t} = \begin{bmatrix}
        \zero{nT\times n} & \eye{nT\times nT}
    \end{bmatrix}$, meaning that \eqref{eq_infmatapxprior} reduces to
    \begin{equation}
        \label{infmatapxprior_later}
        \infmatapxprior{\psliwin{t+1}} = \begin{bmatrix}
            \zero{nT\times n} & \eye{nT\times nT}
        \end{bmatrix} \infmatapx{\sliwin{t}} \begin{bmatrix}
            \zero{n\times nT} \\ \eye{nT\times nT}
        \end{bmatrix}.
    \end{equation}

    By combining \eqref{infmatapx_simple_sum} and \eqref{infmatapxprior_later} and the initial condition
    $\infmatapx{\sliwin{T}} = \infpriorpos{T+1}^\top \infmatapxprior{0:0} \infpriorpos{T+1}
    + \sum_{\ell=1}^{T} \infsubpos{\ell}{T+1}^\top \infsub{t-T+\ell} \infsubpos{\ell}{T+1}$ from \eqref{eq_infmatapx_early},
    we can express $\infmatapx{\sliwin{t}}$ for $t > T$ via
    \begin{multline}
        \label{eq_infmatapx_later}
        \infmatapx{\sliwin{t}} = \infpriorpos{T+1}^\top (\Q{t-T-1}^{-1} + \C{t-T}^\top \Rinv{t-T} \C{t-T}) \infpriorpos{T+1} 
        \\
        + \sum_{\ell=1}^{T} \infsubpos{\ell}{T+1}^\top \infsub{t-T+\ell} \infsubpos{\ell}{T+1}.
    \end{multline}

    Using \eqref{eq_infmatapx_early}, we express $\U{t} \infmatapx{\sliwin{t}} \U{t}^\top$ for $t\in\seq{T-1}$ as
    \begin{multline}
        \label{eq_UinfmatapxU_early}
        \U{t} \infmatapx{\sliwin{t}} \U{t}^\top = \infpriorpos{\T{t+1}}^\top \infmatapxprior{0:0} \infpriorpos{\T{t+1}} 
        \\
        + \sum_{\ell=1}^{\T{t}} \infsubpos{\ell}{\T{t+1}}^\top \infsub{t-\T{t}+\ell} \infsubpos{\ell}{\T{t+1}}.
    \end{multline}

    \ap{We evaluate $\U{t} \infmatapx{\sliwin{t}} \U{t}^\top$ for $t\geq T$ by substituting in $\infmatapx{\sliwin{t}}$ for $t=T$ from \eqref{eq_infmatapx_early} and for $t>T$ from \eqref{eq_infmatapx_later}, both yielding
    \begin{multline}
        \label{eq_UinfmatapxU_later}
        \U{t} \infmatapx{\sliwin{t}} \U{t}^\top
        = \sum_{\ell=1}^{T-1} \infsubpos{\ell}{T}^\top \infsub{t-T+\ell+1} \infsubpos{\ell}{T}
        \\
        +\infpriorpos{T}^\top ( \Qinv{t-T} + \C{t-T+1}^\top \Rinv{t-T+1} \C{t-T+1} ) \infpriorpos{T}. 
    \end{multline}}


    Because we assume that $\infmatestprior{0:0} = \infmatapxprior{0:0}$, we can see that \eqref{eq_UinfmatestU_early} and \eqref{eq_UinfmatestU_later} are identical to \eqref{eq_UinfmatapxU_early} and \eqref{eq_UinfmatapxU_later}.
    As such, we have $\U{t} \infmatest{\sliwin{t}} \U{t}^\top = \U{t} \infmatapx{\sliwin{t}} \U{t}^\top$ for all $t\geq 1$.
\end{proof}

\subsection{Proof of Lemma~\ref{lemma_infmatapx_closed_form}}
\begin{lemma}
    \label{lemma_infmatapx_closed_form}
    Let $\infmatapxprior{\psliwin{t}}$ be initialized with $\infmatapxprior{\psliwin{1}} = \infmatapxprior{0:0}$ and \ap{suppose the} update \ap{laws in} \eqref{eq_infmatapxprior} and \eqref{eq_infmatapx} \ap{are used}.
    Then, the $\superth{i}\superth{j}$ $n\times n$ submatrix of $\infmatapx{\sliwin{t}}$, denoted $\infmatapx{\sliwin{t}}^{\{i\}\{j\}}\in\R{n\times n}$, is defined for all $t \geq 1$ as
    \begin{align}
        & \infmatapx{\sliwin{t}}^{\{i\}\{j\}}
        \label{eq_infmatapprox_closed_form}
        = \begin{cases}
            \infmatapxprior{0:0} + \A{0}^\top \Qinv{0} \A{0}
            & \text{if } i=j=1, t\in\seq{T} \\
            \Qinv{t-1} + \C{t}^\top \Rinv{t} \C{t}
            & \text{if } i=j=\T{t}+1 \\
            \infmatapxcom{q}
            & \text{else if } i=j \\
            -\A{q}^\top \Qinv{q}
            & \text{if } i=j-1 \\
            -\Qinv{q-1} \A{q-1}
            & \text{if } i=j+1 \\
            \zero{n\times n}
            & \text{otherwise}
        \end{cases}
    \end{align}
    for all $i,j \in \{1,\dots, \T{t}+1\}$,
    where $q = t-\T{t}-1+i$ and
    \begin{equation}
        \infmatapxcom{q} = \Qinv{q-1} + \C{q}^\top \Rinv{q} \C{q} + \A{q}^\top \Qinv{q} \A{q}.
    \end{equation}
\end{lemma}
\label{subsect_lemma_infmatapx_closed_form}
\begin{proof}

    \ap{In Appendix~\ref{subsect_lemma_infmatapxprior_equivalent}, we presented $\infmatapx{\sliwin{t}}$ in closed form for $t\in\seq{T}$ in \eqref{eq_infmatapx_early} and for $t>T$ in \eqref{eq_infmatapx_later}.
    By considering both of these closed form calculations, }
    we can express the $\superth{i}\superth{j}$ $n\times n$ submatrix of $\infmatapx{\sliwin{t}}$, 
    denoted $\infmatapx{\sliwin{t}}^{\{i\}\{j\}}\in\R{n\times n}$, for all $t \geq 1$ as
    \begin{align}
        & \infmatapx{\sliwin{t}}^{\{i\}\{j\}}
        = \begin{cases}
            \infmatapxprior{0:0} + \A{0}^\top \Qinv{0} \A{0}
            & \text{if } i=j=1, t\in\seq{T} \\
            \Qinv{t-1} + \C{t}^\top \Rinv{t} \C{t}
            & \text{if } i=j=\T{t}+1 \\
            \infmatapxcom{q}
            & \text{else if } i=j \\
            -\A{q}^\top \Qinv{q}
            & \text{if } i=j-1 \\
            -\Qinv{q-1} \A{q-1}
            & \text{if } i=j+1 \\
            \zero{n\times n}
            & \text{otherwise}
        \end{cases}
    \end{align}
    for all $i,j \in \{1,\dots, \T{t}+1\}$,
    where $q = t-\T{t}-1+i$ and
    \begin{equation}
        \infmatapxcom{q} = \Qinv{q-1} + \C{q}^\top \Rinv{q} \C{q} + \A{q}^\top \Qinv{q} \A{q}.
    \end{equation}    
\end{proof}

\subsection{Proof of Lemma~\ref{lemma_depsetapx}}
\label{subsect_lemma_depsetapx}
\begin{proof}
    \ap{Using the definition of $\priorpropmat{t}$ from \eqref{eq_alg_prior_propagate} and the definition of $\depsetapxtir{t}{i}{\A{}}$ from \eqref{eq_depsets_asynch}, it follows that
    \begin{equation}
        \depsetapxtir{t}{i}{\A{}} = i \cup \depop{i}{\A{t-1}}.
    \end{equation}}

    \ap{Using the definition of $\coeffmaty{t}$ from \eqref{eq_coeffmaty} and the definition of $\depsetapxtir{t}{i}{\y{}}$ from \eqref{eq_depsets_asynch}, it follows that
    \begin{equation}
        \depsetapxtir{t}{i}{\y{}} = \depop{i}{\C{t}^\top \Rinv{t}}.
    \end{equation}}

    \ap{Using the definition of $\depsetapxtir{t}{i}{\xapx{}}$ from \eqref{eq_depsets_asynch} with \eqref{eq_infmatapx} yields $\depsetapxtir{t}{i}{\xapx{}} = \depop{i}{\infmatapx{\sliwin{t}}}$.
    Combining this with the closed form representation of $\infmatapx{\sliwin{t}}$ from \eqref{eq_infmatapprox_closed_form} yields
    \begin{equation}
        \depsetapxtir{t}{i}{\xapx{}} = \depsetapxZti{t}{i} \cup \bigcup_{q=t-\T{t}}^{t-1} \depsetapxtiq{t}{i}{q},
    \end{equation}
    \ap{where $\depsetapxZti{t}{i}$ and $\depsetapxtiq{t}{i}{q}$ are from \eqref{eq_def_depsetapxZti} and \eqref{eq_def_depsetapxtiq} respectively.}}

    \ap{Using the definition of $\coeffmatxapxprior{t}$ from \eqref{eq_coeffmatxapxprior} and the definition of $\depsetapxtir{t}{i}{\xapxprior{}}$ from \eqref{eq_depsets_asynch} yields $\depsetapxtir{t}{i}{\xapxprior{}} = \depop{i}{\infmatapxprior{\psliwin{t}}}$.
    By combining the initial condition $\infmatapxprior{\psliwin{1}} = \infmatapxprior{0:0}$ with \eqref{eq_infmatapprox_closed_form} and \eqref{eq_infmatapxprior}, we can compute the $\superth{i}\superth{j}$ $n\times n$ submatrix of $\infmatapxprior{\psliwin{t}}$, denoted $\infmatapxprior{\psliwin{t}}^{\{i\}\{j\}}\in\R{n\times n}$, as
    \begin{align}
        & \infmatapxprior{\psliwin{t}}^{\{i\}\{j\}}
        \label{eq_infmatapproxprior_closed_form}
        \\
        &= \begin{cases}
            \infmatapxprior{0:0}
            & \text{if } i=j=1, t = 1 \\
            \infmatapxprior{0:0} + \A{0}^\top \Qinv{0} \A{0}
            & \text{if } i=j=1, t\in\{2,\dots,T\} \\
            \Qinv{t-2} + \C{t-1}^\top \Rinv{t-1} \C{t-1}\hspace{-5pt}\null
            & \text{if } i=j=\T{t} \\
            \infmatapxcom{q}
            & \text{else if } i=j \\
            -\A{q}^\top \Qinv{q}
            & \text{if } i=j-1 \\
            -\Qinv{q-1} \A{q-1}
            & \text{if } i=j+1 \\
            \zero{n\times n}
            & \text{otherwise}
        \end{cases}
    \end{align}
    for all $i,j \in \{1,\dots, \T{t}\}$,
    where $q = t-\T{t}-1+i$ and
    \begin{equation}
        \infmatapxcom{q} = \Qinv{q-1} + \C{q}^\top \Rinv{q} \C{q} + \A{q}^\top \Qinv{q} \A{q}.
    \end{equation}
    As such, we can express $\depsetapxtir{t}{i}{\xapxprior{}}$ as
    \begin{equation}
        \depsetapxtir{t}{i}{\xapxprior{}} = \depsetapxZti{t}{i} \cup \bigcup_{q=t-\T{t}}^{t-2} \depsetapxtiq{t}{i}{q},
    \end{equation}
    \ap{where $\depsetapxZti{t}{i}$ and $\depsetapxtiq{t}{i}{q}$ are from \eqref{eq_def_depsetapxZti} and \eqref{eq_def_depsetapxtiq} respectively.}}
\end{proof}

\subsection{Proof of Lemma~\ref{lemma_error_bound}}
\label{subsect_lemma_error_bound}
\begin{lemma}
    \label{lemma_error_bound}
    Let $\Japx$ be defined as in \eqref{eq_obj_func_approx_quad} and $\x{\sliwin{t}}^*$ be the unique minimizer of $\Japx$.
    For any $\x{\sliwin{t}}\in\R{n(\T{t}+1)}$, the error bound condition holds, i.e.,
    \begin{equation}
        \norm{\x{\sliwin{t}} - \x{\sliwin{t}}^*} \leq \kappa \norm{\nabla \Japx(\x{\sliwin{t}})},
    \end{equation}
    where $\kappa = \frac{1}{\lambda_{\min}(\Kapx{t})} > 0$.
\end{lemma}
\begin{proof}
    As shown in Lemma~\ref{lemma_asynch_analytical_solution},
    the objective function $\Japx$ in \eqref{eq_obj_func_approx_quad} has a unique solution given by 
    $\x{\sliwin{t}}^* = -\Kapx{t}^{-1}\fapx{t}$.
    Consider the distance of $\x{\sliwin{t}}^*$ from some $\x{\sliwin{t}}\in\R{n(\T{t}+1)}$, given by $\norm{\x{\sliwin{t}} - \x{\sliwin{t}}^*}$.
    Pre-multiplying $\x{\sliwin{t}}$ with $\Kapx{t}^{-1}\Kapx{t}$ and substituting $\x{\sliwin{t}}^*$ with $-\Kapx{t}^{-1}\fapx{t}$ yields
    \begin{equation}
        \norm{\x{\sliwin{t}} - \x{\sliwin{t}}^*} = \norm{\Kapx{t}^{-1}\Kapx{t}\x{\sliwin{t}} -\Kapx{t}^{-1}\fapx{t}}.
    \end{equation}
    We factor the inner term to yield
    \begin{equation}
        \norm{\x{\sliwin{t}} - \x{\sliwin{t}}^*} = \norm{\Kapx{t}^{-1}(\Kapx{t}\x{\sliwin{t}} + \fapx{t})}
    \end{equation}
    We use a property of the matrix norm to yield,
    \begin{equation}
        \norm{\x{\sliwin{t}} - \x{\sliwin{t}}^*} \leq \norm{\Kapx{t}^{-1}}\norm{\Kapx{t}\x{\sliwin{t}} + \fapx{t}}.
    \end{equation}
    We define $\kappa = \norm{\Kapx{t}^{-1}} = \frac{1}{\lambda_{\min}(\Kapx{t})}$ and recall that $\nabla \Japx(\x{\sliwin{t}}) = \Kapx{t}\x{\sliwin{t}} + \fapx{t}$.
    As such, making these substitutions yields the desired inequality,
    \begin{equation}
        \norm{\x{\sliwin{t}} - \x{\sliwin{t}}^*} \leq \kappa \norm{\nabla \Japx(\x{\sliwin{t}})}.
    \end{equation}
\end{proof}


\subsection{Proof of Lemma~\ref{lemma_lipschitz_continuous}}
\label{subsect_lemma_lipschitz_continuous}
\begin{lemma}
    \label{lemma_lipschitz_continuous}
    The gradient of the function $\Japx$ in \eqref{eq_obj_func_approx_quad} is $L$-Lipschitz continuous on any compact interval, with $L=\lambda_{\max}(\Kapx{t})$.
\end{lemma}
\begin{proof}
    By definition of the gradient of \eqref{eq_obj_func_approx_quad}, we know
    \begin{equation}
        \norm{\nabla \Japx(\x{}) - \nabla \Japx(\y{})}
        = \norm{\Kapx{t}\x{} + \fapx{t} - \Kapx{t}\y{} - \fapx{t}}.
    \end{equation}
    
    This can be simplified and factored to
    \begin{equation}
        \norm{\nabla \Japx(\x{}) - \nabla \Japx(\y{})}
        = \norm{\Kapx{t}(\x{} - \y{})}.
    \end{equation}

    Using a property of the matrix norm, we can split the norm on the right, according to
    \begin{equation}
        \norm{\nabla \Japx(\x{}) - \nabla \Japx(\y{})}
        \leq \norm{\Kapx{t}}\norm{(\x{} - \y{})}.
    \end{equation}

    We define $L=\norm{\K{t}}=\lambda_{\max}(\Kapx{t})$.
    Substituting this into the expression yields
    \begin{equation}
        \norm{\nabla \Japx(\x{}) - \nabla \Japx(\y{})}
        \leq L\norm{\x{} - \y{}}.
    \end{equation}

    Because $\Kapx{t}$ is positive definite, then $L=\lambda_{\max}(\Kapx{t})>0$.
\end{proof}

\subsection{Proof of Theorem~\ref{theorem_convergence_rate}}
\label{subsect_theorem_convergence_rate}
\begin{proof}
    In order to use Proposition 2.2 of \cite{tseng1991rate}, we must first prove that
    minimizing $\Japx$ using Algorithm~\ref{alg_bcd_asynch}
    satisfies specific assumptions.
    
    First, we recall that our objective function, given by \eqref{eq_obj_func_approx_quad}, is an unconstrained quadratic function with a positive definite Hessian and a global minimizer 
    $\x{\sliwin{t}}^* = -\Kapx{t}^{-1}\fapx{t}$, as proven in \ap{Appendix~\ref{subsect_lemma_asynch_analytical_solution}.}
    This proves that $\Japx$ is lower bounded and has a nonempty solution set.
    Additionally, $\Japx$ is $L$-Lipschitz continuous with $L=\lambda_{\max}(\Kapx{t})$ as proven in Lemma~\ref{lemma_lipschitz_continuous}.

    Additionally, while \cite{tseng1991rate} only requires the error bound hold locally for small enough $\Japx(\xtotali{\sliwin{t}}{k})$ and $\norm{\nabla \Japx(\xtotali{\sliwin{t}}{k})}$, we proved in Lemma~\ref{lemma_error_bound} that the error bound condition holds for $\Japx$ globally.

    We need $rB \geq \hat{k}$ such that the error bound condition holds and such that the closest minimizer of $\Japx$ to $\xtotali{\sliwin{t}}{k}$ does not change.
    However, we note that the error bound condition holds for all $\xtotali{\sliwin{t}}{k}$ and that the minimizer of $\Japx$ is always $\x{\sliwin{t}}^* = -\Kapx{t}^{-1}\fapx{t}$.
    As such, $\hat{k} = 0$.
    Additionally, our partial asynchrony assumptions satisfy those of \cite{tseng1991rate}.
    As such, the desired result follows by Proposition 2.2 of \cite{tseng1991rate}. 
\end{proof}

\end{document}